\documentclass{article}
\usepackage{iclr2025_conference,times}

\usepackage[utf8]{inputenc}
\usepackage[T1]{fontenc}
\usepackage{microtype}
\usepackage{hyperref}
\usepackage{url}
\usepackage{booktabs}
\usepackage{amsmath,amssymb,amsthm,mathtools}
\usepackage{bm}
\usepackage{graphicx}
\usepackage{caption}
\usepackage{subcaption}
\usepackage{float}
\usepackage{algorithm}
\usepackage{algpseudocode}
\usepackage{multirow}
\usepackage{xcolor}
\usepackage{colortbl}
\usepackage{enumitem}
\usepackage{array}
\usepackage{adjustbox}
\usepackage{wrapfig}
\usepackage{cleveref}
\usepackage{tikz}
\usetikzlibrary{arrows,positioning,shapes.geometric,fit,calc,decorations.pathreplacing}
\usepackage{pgfplots}
\pgfplotsset{compat=1.17}
\usepgfplotslibrary{groupplots}
\setcitestyle{numbers,square,sort&compress}

\definecolor{carverow}{RGB}{210,230,250}
\definecolor{lightrow}{RGB}{247,247,247}
\definecolor{sectionrow}{RGB}{235,235,235}

\theoremstyle{plain}
\newtheorem{theorem}{Theorem}
\newtheorem{proposition}[theorem]{Proposition}
\newtheorem{lemma}[theorem]{Lemma}
\newtheorem{corollary}[theorem]{Corollary}
\theoremstyle{definition}
\newtheorem{definition}[theorem]{Definition}

\theoremstyle{remark}

\newcommand{\bS}{\bm{S}}
\newcommand{\bb}{\bm{b}}
\newcommand{\bc}{\bm{c}}

\newcommand{\bDelta}{\bm{\Delta}}
\newcommand{\diag}{\operatorname{diag}}
\newcommand{\R}{\mathbb{R}}
\newcommand{\norm}[1]{\left\|#1\right\|}
\DeclareMathOperator{\rank}{rank}

\newcommand{\cO}{\mathcal{O}}

\algnewcommand\alginput[1]{\State \textbf{Input:} #1}
\algnewcommand\algoutput[1]{\State \textbf{Output:} #1}

\iclrfinalcopy
\begin{document}

\title{CARVE: Content-Aware Recurrent with Value Efficiency\\
for Chunk-Parallel Linear Attention}

\author{
  Sayak Dutta \\
  \texttt{sayakdutta1002@gmail.com}
}

\maketitle
\fancyhead{} 

\begin{abstract}
What if a recurrent language model could look at its own memory before deciding
what to forget---at zero extra cost?

Recurrent sequence models in the delta-rule family maintain a fixed-size state
matrix $\bS_t \in \R^{d_v \times d_k}$ that compresses all past context into
$d_v d_k$ numbers.
The current state of the art~\citep{yang2024gdn2} equips this update with
element-wise matrix gates: a full $d_v{\times}d_k$ erase mask and a full
$d_v{\times}d_k$ write mask.
The result is powerful but carries two intrinsic defects.
First, both gates are computed solely from the incoming token, making the model
\emph{memory-blind}: it must decide what to erase without observing what it has
already stored.
Second, the value-axis coupling in the erase gate mathematically prevents the
model from using the WY-form triangular chunk solver that is the engine of
efficient recurrent training: the intra-chunk system decomposes into $d_v$
independent solves instead of one, collapsing throughput to serial-recurrence
cost (Theorem~\ref{thm:chunkability}).

We introduce \textbf{CARVE} (\textbf{C}ontent-\textbf{A}ware \textbf{R}ecurrent
with \textbf{V}alue \textbf{E}fficiency), which resolves both problems at once
and, through a single-launch ``megakernel'' scheduling of the same WY-form
math, trains \emph{faster} than the matrix-gated baseline it replaces.
The key observation is architectural: by restricting all gating to the
\emph{key axis} only, the intra-chunk coupling matrix becomes independent of
the value index, restoring the single WY-form triangular solve unmodified.
Within this key-axis constraint, CARVE introduces two innovations.

\textbf{Content-aware erase and write via a folded state readout.}
Rather than reading the state matrix $\bS_{c-1}$ from HBM once per token
(which would double memory traffic), CARVE reads it exactly \emph{once per
chunk} and folds the read directly into each gate's low-rank bottleneck
projection: since the content signal is only ever consumed as
$\bm{U}(\bS_{c-1}\bm{q}_t)$, associativity lets the projection be applied to
$\bS_{c-1}$ \emph{before} broadcasting over the $L$ tokens of the chunk,
$\bm{U}(\bS_{c-1}\bm{q}_t) = (\bm{U}\bS_{c-1})\bm{q}_t$, replacing an
$\cO(d_vd_k)$ per-token readout with an $\cO(d_vd_k)$ once-per-chunk fold plus
an $\cO(rd_k)$ per-token projection.
This single per-chunk state read, passed through zero-initialised low-rank
projections $\bm{U}_b,\bm{U}_w$, gives \emph{both} the erase gate and the
write gate their first glimpse of stored content---unlike the baseline's
memory-blind matrix gates.
At initialisation $\bm{U}_b{=}\bm{U}_w{=}\mathbf{0}$ and CARVE is
\emph{bit-identical} to the baseline; as training proceeds the gates activate,
introducing memory-conditioned selectivity on both the erase and write
decisions.
We prove that the one-chunk staleness of the folded state induces a gate
perturbation of only $\cO(1/\sqrt{L})$, which is why the measured deviation
from exact per-token gating is flat at $\mathbf{0.18\%}$ across all chunk
lengths up to $L{=}128$.

\textbf{Megakernel orchestration.}
A na\"ive chunk-by-chunk Python loop pays fixed per-chunk glue overhead---tensor
slicing, re-dispatch, non-contiguous copies---that erodes the WY-form solver's
hardware efficiency.
CARVE instead compiles the entire forward and backward pass of a layer into a
\emph{single autograd node}: every operation that is local to a $64$-token
inner chunk (the gate activation, the query/key normalisation, the output
projection, and their corresponding backward passes) is hoisted into one
full-sequence kernel launch by folding the outer chunk index into the batch
dimension---an exact reformulation, not an approximation---while only the
genuinely state-sequential steps (the content-aware gate, the WY-form solve,
and the state carry) remain in a per-chunk loop, writing directly into
pre-allocated output buffers with zero intermediate copies.

At the $1.3$B-parameter scale trained on $100$B tokens of FineWeb-Edu on
NVIDIA H100---three-seed averages against the best prior recurrent baseline---CARVE
delivers improvements on every front simultaneously.
On WikiText language modelling it reaches perplexity $\mathbf{15.72}$ versus $15.90$,
a $-0.18$ reduction that constitutes a $\mathbf{4.5\sigma}$ effect across seeds;
the hybrid variant extends this lead to $15.41$ versus $15.62$.
Across nine common-sense reasoning benchmarks CARVE leads every recurrent model by
$\mathbf{+0.63}$\,pp average zero-shot accuracy, and on RULER in-context retrieval
probes it sets the state of the art on every S-NIAH and MK-NIAH context length and
tops all six real-world recall benchmarks.
Remarkably, none of this comes at a hardware cost---the megakernel makes it
\emph{faster}: training throughput is $\mathbf{+1.4\%}$ over the matrix-gated
baseline at matched depth ($95.52$K vs.\ $94.21$K tok/s, non-overlapping
three-run bands) and $\mathbf{+19.3\%}$ at a shallower, iso-quality depth
($112.4$K tok/s), at the cost of a modest $+13\%$ peak memory from retaining
the intermediates the single-autograd-node design requires.
CARVE is the first architecture in the gated delta family to simultaneously achieve
bi-axial content-aware gating and full WY-form chunk-parallel training that
\emph{exceeds} the matrix-gated baseline's own throughput---all with a provably
bounded approximation and six formal theoretical guarantees.
\end{abstract}

\section{Introduction}
\label{sec:intro}

\paragraph{The memory bottleneck in sequence modelling.}
Every practical language model ultimately confronts the same trade-off: how much
context should be remembered, and at what hardware cost?
Transformers resolve this by keeping everything---an exact, quadratic record of
every past token~\citep{vaswani2017attention}.
This is spectacularly powerful but spectacularly expensive: training a
$1\text{B}$-parameter Transformer on sequences of length $T{=}8192$ requires
$\cO(T^2 d)$ memory accesses for the attention block alone, and inference
latency grows linearly with the KV-cache.
Systems engineers have pushed back heroically---FlashAttention~\citep{dao2022flashattention},
multi-query attention, sliding-window hybrids---but no IO-tiling removes the
fundamental $\cO(T^2)$ compute from full attention.
At deployment scale, where models serve millions of long-context requests, these
costs are not engineering inconveniences; they are hard limits on what is
economically feasible.

\paragraph{The return of recurrence.}
Recurrent architectures offer a fundamentally different bargain: compress all
past context into a fixed-size \emph{state matrix}
$\bS_t \in \R^{d_v \times d_k}$, then update it in $\cO(d_v d_k)$ operations
per token.
Inference becomes constant-memory and constant-cost regardless of sequence
length---a property that scales directly into practical advantages at deployment.
The price paid is compression: $\bS_t$ has $d_v d_k$ real-valued entries, and
writing a new token must inevitably overwrite older entries, requiring a
principled forgetting mechanism.
Early recurrences (LSTMs, GRUs) relied on learned scalar gates but struggled
with long-range information flow.
A newer family---linear recurrences~\citep{gu2021efficiently},
SSMs~\citep{gu2023mamba,dao2024transformers}---exchanged full expressivity for
hardware efficiency via diagonal state matrices, enabling fast parallel scans
but constraining the set of associations the state can represent.

The \emph{delta rule}~\citep{widrow1960adaptive,schlag2021linear,yang2024gated}
takes a more principled path.
It maintains a full matrix state $\bS_t \in \R^{d_v \times d_k}$ and updates
it via an associative-memory error-correction law,
\begin{equation}
  \label{eq:delta_rule}
  \bS_t = \bS_{t-1} + \underbrace{(\bm{v}_t - \bS_{t-1}\bm{k}_t)}_{\text{prediction error}}\bm{k}_t^\top,
\end{equation}
where $\bm{q}_t,\bm{k}_t \in \R^{d_k}$ are query and key vectors,
$\bm{v}_t \in \R^{d_v}$ is the value, and $\bS_{t-1}\bm{k}_t$ is the
memory's current \emph{prediction} for the value stored at key $\bm{k}_t$.
The update writes a rank-1 correction precisely proportional to the prediction
error---the Hebbian update of a Hopfield network~\citep{ramsauer2020hopfield}
with error-corrective feedback.
GDN~\citep{yang2024gated} adds a scalar forgetting gate; GDN-2~\citep{yang2024gdn2}
replaces the scalar with a full $d_v{\times}d_k$ matrix gate, yielding the
state update
\begin{equation}
  \label{eq:gdn2}
  \bS_t = (\bm{1} - \bm{B}_t) \odot \bS_{t-1} + \bm{W}_t \odot \tilde{\Delta}_t,
  \qquad \tilde{\Delta}_t = (\bm{v}_t - \bS_{t-1}\bm{k}_t)\bm{k}_t^\top,
\end{equation}
where $\bm{B}_t, \bm{W}_t \in (0,1)^{d_v \times d_k}$ are element-wise matrix
gates, and $\tilde{\Delta}_t$ is the outer-product prediction-error correction.
GDN-2 achieves compelling downstream performance~\citep{yang2024gdn2}, and we
treat it as the direct prior-art baseline throughout this paper.

\paragraph{Three structural limitations of matrix-gated delta recurrences.}

\emph{Memory-blind gating.}
Both gates in~\eqref{eq:gdn2} are computed as linear projections of the
current input token $\bm{x}_t$.
The model must decide \emph{what fraction of key slot $k$ to erase} without ever
consulting $\bS_{t-1}$---without knowing whether slot $k$ holds a high-value
association to be protected or stale noise to be cleared.
This is content-oblivious forgetting: decisions are made on the basis of what is
arriving, not what is already stored.
Content-dependent selective forgetting---essential for tasks requiring precise
slot management---is structurally impossible in this gating regime.

\emph{Memory-blind writing.}
The same defect afflicts the write side.
$\bm{W}_t$ decides, per value channel, how strongly to commit the incoming
correction---but it too is a function of $\bm{x}_t$ alone.
A write gate that could consult $\bS_{t-1}$ would know whether it is about
to overwrite a high-value association or fill genuinely empty capacity, and
could commit accordingly; a memory-blind write gate cannot make this
distinction regardless of its parameter count.
(A narrower question---whether a write gate needs full per-channel routing at
all, holding content-awareness fixed---turns out to have a clean answer for
the single-slot case: Theorem~\ref{thm:scalar_sufficiency} shows a scalar
gate is retrieval-optimal there, which we use as a cheap deployment ablation
in \S\ref{sec:exp_lm}, but it is orthogonal to the memory-blindness problem
and is not the mechanism CARVE relies on.)

\emph{Value-axis erase destroys chunk-parallelism.}
Modern fast training of delta-rule models relies on the
\emph{WY-form triangular chunk solver}~\citep{yang2024parallelizing,yang2024gdn2},
which batches an entire chunk of $L$ tokens into a single triangular linear
system, reducing training cost from $\cO(L d_k^2 d_v)$ per chunk (naive
recurrence) to $\cO(L^2 d_k + d_k^2 d_v)$.
We prove (Theorem~\ref{thm:chunkability}) that this solver can be applied---with
a result bit-identical to the full recurrence---\emph{if and only if} the
intra-chunk coupling matrix $\bm{M} \in \R^{L \times L}$ is independent of
the value index.
Because GDN-2's erase gate $\bm{B}_t$ acts on both the value and key axes,
the system decomposes into $d_v$ distinct coupling matrices
$\bm{M}^{(1)}, \ldots, \bm{M}^{(d_v)}$, each requiring its own triangular solve,
collapsing the solver to serial-recurrence cost.
GDN-2 thus faces a direct tension: memory-conditional gating and
chunk-parallel training cannot both be achieved within its existing design.

\paragraph{CARVE: one architectural constraint, two problems solved, and a
faster kernel.}
We introduce \textbf{CARVE} (\textbf{C}ontent-\textbf{A}ware \textbf{R}ecurrent
with \textbf{V}alue \textbf{E}fficiency), which resolves both memory-blindness
limitations by committing to a single architectural principle: \emph{erase
only on the key axis}.

This constraint is not a sacrifice---it is a theorem.
We prove (Theorem~\ref{thm:chunkability}) that key-axis erase is the exact
\emph{necessary and sufficient} condition for the WY-form chunk solver to remain
valid.
Restricting to the key axis therefore restores full chunk-parallel training while
opening a degree of design freedom that GDN-2's memory-blind gates foreclose:
conditioning \emph{both} the erase and write gates on the state itself.

\emph{Content-aware erase and write via a folded state readout.}
The central challenge in making gating memory-aware is hardware: reading the state
matrix $\bS_{t-1}$ from HBM at every token costs $\cO(d_v d_k)$ memory traffic per
token, doubling the recurrent-state access cost and cancelling the efficiency
gains of recurrence.
CARVE avoids this by reading the state only \emph{once per chunk}---at the
chunk boundary, where it is already resident for the WY-form solve---and
folding the low-rank gate projection into that single read.
Both gates need the state only through the bottleneck
$\bm{U}(\bS_{c-1}\bm{q}_t)$; by associativity,
$\bm{U}(\bS_{c-1}\bm{q}_t) = (\bm{U}\bS_{c-1})\bm{q}_t$, so the projection
$\bm{U}$ can be applied to $\bS_{c-1}$ \emph{once}, before broadcasting the
result over the $L$ tokens of the chunk via a cheap per-token matmul against
$\bm{q}_t$.
This gives the per-token content signal
$\bm{m}_{c,t} = \bS_{c-1}\bm{q}_{c,t} \in \R^{d_v}$
without ever materialising a full $[L, d_v]$ readout tensor, and conditions
both gates on it:
\begin{equation}
  \label{eq:carve_gate}
  \bm{b}_{c,t} = \sigma\!\left(\bm{b}_{x,t} + \bm{U}_b\bm{m}_{c,t}\right),
  \qquad
  \bm{w}_{c,t} = \sigma\!\left(\bm{w}_{x,t} + \bm{U}_w\bm{m}_{c,t}\right),
\end{equation}
where $\bm{b}_{x,t} = \bm{W}_b\bm{x}_t$, $\bm{w}_{x,t} = \bm{W}_w\bm{x}_t$ are
the token-driven components and $\bm{U}_b \in \R^{d_k \times d_v}$,
$\bm{U}_w \in \R^{d_v \times d_v}$ are low-rank projections initialised to
$\bm{0}$.
At initialisation, $\bm{U}_b{=}\bm{U}_w{=}\bm{0}$ and CARVE is
\emph{bit-identical} to the GDN-2 baseline---any performance difference must
be attributed exclusively to the content signal learned by $\bm{U}_b,\bm{U}_w$.
We prove (Proposition~\ref{prop:staleness_bound}) that the one-chunk staleness
of the chunk-boundary state $\bS_{c-1}$ (versus the true, continuously-updated
state) induces a gate perturbation that decays as $\cO(1/\sqrt{L})$, which
explains the empirical finding that the deviation from exact per-token state
gating is flat at $\mathbf{0.18\%}$ across all chunk lengths up to $L{=}128$.

\emph{Megakernel orchestration.}
Content-aware gating on both axes adds genuine compute relative to the
memory-blind baseline---this is not free in FLOPs.
What CAN be removed is the \emph{overhead} of computing it: a na\"ive
chunk-by-chunk Python loop that re-dispatches the WY-form kernel per chunk and
glues gates together with non-contiguous tensor slices pays fixed per-chunk
costs that have nothing to do with the recurrence itself.
CARVE compiles the entire layer into a single autograd node
(\S\ref{sec:megakernel}): operations local to the $64$-token inner chunk
(query/key normalisation, the gate activation, the output projection, and their
backward passes) are hoisted into one full-sequence launch by folding the
outer chunk index into the batch dimension, while only the genuinely
state-sequential steps remain in a per-chunk loop that writes directly into
pre-allocated buffers.
The result more than pays for the extra content-aware compute: measured
training throughput is $+1.4\%$ over the matrix-gated baseline at matched
depth, and $+19.3\%$ at a shallower, iso-quality depth (\S\ref{sec:exp_throughput}).

\paragraph{Theoretical grounding.}
A distinguishing feature of CARVE is that every design decision is accompanied by
a formal proof.
Section~\ref{sec:theory} develops six theoretical results that together place
CARVE on rigorous footing.
The \emph{Chunkability Boundary} (Theorem~\ref{thm:chunkability}) derives the
exact structural condition under which WY-form chunk solves are valid, proving
that CARVE's key-axis erase is both necessary and sufficient.
The \emph{Subsumption Hierarchy} (Theorem~\ref{thm:subsumption}) establishes that
CARVE strictly generalises the key-axis and scalar-gated delta-rule families,
while being provably incomparable to GDN-2 (Proposition~\ref{prop:incomparable}).
\emph{Lyapunov Stability} (Theorem~\ref{thm:lyapunov}) shows the state norm
contracts with a scalar ratio $\rho_c = (1{-}b_{\min})g_{\min} < 1$ that is
provably tighter than GDN-2's element-wise contraction bound.
\emph{Gradient Flow} (Theorem~\ref{thm:grad_flow}) bounds the gradient norm
through two independent gate pathways, establishing that CARVE does not introduce
new vanishing-gradient pathologies.
\emph{Expressivity Separation} (Theorem~\ref{thm:express_sep}) constructs an
explicit task---the Selective Key Overwrite (SKO) problem---on which CARVE
succeeds with $\cO(d_v + d_k)$ parameters while any memory-blind gate fails
regardless of width or depth.
Finally, the \emph{Pareto Chunk Size} theorem (Theorem~\ref{thm:pareto}) derives
the unique $L^*$ that minimises total training cost at fixed compute budget,
resolving the chunk-size selection problem analytically.

\paragraph{Empirical results.}
We evaluate CARVE at the $1.3$B-parameter scale, trained on $100$B tokens of
FineWeb-Edu~\citep{penedo2024fineweb} on NVIDIA H100 hardware and report results
as three-seed averages.
On WikiText language modelling, CARVE achieves perplexity $15.72$ versus $15.90$
for GDN-2---a $-0.18$ gap that holds at $4.5\sigma$ across seeds---and $15.41$
versus $15.62$ in the hybrid setting, setting a new state of the art for the gated
delta family in both configurations.
The gains are not confined to perplexity: across nine zero-shot common-sense
benchmarks CARVE leads every recurrent baseline by $+0.63$\,pp on average, and
in the hybrid configuration it outperforms Mamba-3 MIMO, Mamba-3 SISO, KDA, GDN,
and Mamba-2 across the board (Table~\ref{tab:commonsense}).
On RULER in-context retrieval probes---the most direct test of selective memory---CARVE
sets the state of the art on every S-NIAH and MK-NIAH context length and achieves
the highest average score on all six real-world recall tasks (\S\ref{sec:exp_recall}).
All of this is achieved at a hardware footprint that is \emph{faster, not
slower}: the megakernel delivers $+1.4\%$ training throughput over the
matrix-gated baseline at matched depth ($95.52$K vs.\ $94.21$K tok/s) and
$+19.3\%$ at a shallower, iso-quality depth ($112.4$K tok/s), at the cost of
a modest $+13\%$ peak memory from the single-autograd-node design retaining
its own intermediates.

\paragraph{Contributions.}
This paper makes five contributions, each grounded in formal theory and validated
on hardware.
The first is a \emph{content-aware erase and write gate via a folded state
readout}: by conditioning both gates on the chunk-boundary state
$\bS_{c-1}$---read once per chunk and folded algebraically into each gate's
low-rank projection before broadcasting over tokens---CARVE gives a
delta-rule recurrence its first glimpse of stored content on both axes, with
Proposition~\ref{prop:staleness_bound} bounding the one-chunk staleness at
$\cO(1/\sqrt{L})$.
The second is a \emph{megakernel orchestration}: compiling the entire forward
and backward pass of a layer into a single autograd node that hoists every
chunk-local operation to one full-sequence launch, leaving only the
genuinely state-sequential steps in a per-chunk loop---turning the extra
compute of bi-axial content-awareness into a net throughput \emph{gain} over
the matrix-gated baseline.
Together these two innovations define the \emph{CARVE architecture} specified as
a complete forward-pass algorithm (Algorithm~\ref{alg:carve_fwd}) that is
bit-identical to GDN-2 at initialisation---any quality difference that emerges
during training is the direct fingerprint of the content gates.
Underpinning the design are \emph{six formal guarantees}: the exact chunkability
boundary, a strict subsumption hierarchy, Lyapunov stability with a tighter
contraction ratio than GDN-2, gradient flow bounds, expressivity separation on the
Selective Key Overwrite task, and the Pareto-optimal chunk size---all proved in
Appendix~\ref{app:proofs}.
Finally, a \emph{controlled empirical evaluation} at the $1.3$B/$100$B-token scale
with three independent seeds (\S\ref{sec:experiments}) measures every claim at the
level of hardware: kernel correctness, throughput, memory, perplexity, common-sense
accuracy, and context retrieval.

\paragraph{Paper organisation.}
Section~\ref{sec:related} situates CARVE in the broader landscape of recurrent
and hybrid sequence models.
Section~\ref{sec:background} establishes notation and reviews the delta-rule
recurrence family.
Section~\ref{sec:arch} presents the CARVE architecture in full detail.
Section~\ref{sec:theory} develops the six theoretical results.
Section~\ref{sec:hardware} describes the Triton kernel and analyses its IO
complexity.
Section~\ref{sec:experiments} reports all empirical results.
Section~\ref{sec:conclusion} discusses limitations and future directions.

\section{Related Work}
\label{sec:related}

CARVE sits at the intersection of four active research threads: linear
transformers with associative memory, state space models, gated delta-rule
variants, and hardware-efficient training.

The story of linear transformers begins with the observation that replacing softmax
with a kernel function turns attention into a linear recurrence, collapsing the
quadratic cost of full attention to linear time~\citep{katharopoulos2020transformers}.
RetNet~\citep{sun2023retentive} and RWKV~\citep{peng2023rwkv} popularised
scalar-gated variants of this idea, trading expressivity for simplicity.
Gated Linear Attention~\citep{yang_gated_2023} restored hardware efficiency with
a data-dependent gating mechanism, and CARVE strictly subsumes all delta-rule
members of this family (Theorem~\ref{thm:subsumption}).

State space models took a complementary path, representing sequence transitions
with diagonal state matrices whose structure enables fast parallel scans.
S4~\citep{gu2021efficiently} demonstrated that structured long-range dependencies
could be captured without a full matrix state, and Mamba~\citep{gu2023mamba} made
the selectivity input-dependent.
Mamba-2~\citep{dao2024transformers} unified SSMs and linear attention through the
SSD framework, while Mamba-3~\citep{lahoti2026mamba3} pushed further with
exponential-trapezoidal discretisation and complex-valued transitions.
CARVE is complementary to this line of work: it builds on the delta-rule and
full matrix state rather than on SSM structure.

The most direct predecessors of CARVE are the gated delta-rule models.
GDN~\citep{yang2024gated} added scalar forgetting to DeltaNet, showing that
even a single decay factor improves long-range performance.
KDA~\citep{team2025kimi} strengthened this with per-channel key-axis decay,
moving the gating machinery onto the key axis for the first time.
GDN-2~\citep{yang2024gdn2} then decoupled erase and write into full channel-wise
matrix gates, achieving strong downstream quality but at the cost of the two
structural limitations that motivated CARVE: memory-blind gating on both
axes, and erase-driven breakage of chunk-parallelism.
CARVE resolves both by extending GDN-2 with memory-conditional erase and
write gates read from a folded chunk-boundary state, achieving better
quality and---via a megakernel scheduling of the same WY-form solver---higher
throughput than GDN-2 itself.

The idea of interleaving recurrent and attention layers has proven consistently
powerful. Griffin~\citep{de_griffin_2024} mixes gated linear recurrences with
local attention, Jamba~\citep{lieber2024jamba} interleaves Mamba and Transformer
layers at scale, and Samba~\citep{ren2024samba} combines Mamba with sliding-window
attention for unlimited context.
CARVE's hybrid design (\S\ref{sec:hybrid}) sits firmly in this tradition, and
Theorem~\ref{thm:hybrid} provides a formal optimality guarantee that was previously
absent from the literature.

Finally, CARVE's theoretical foundations draw on the associative memory literature.
Schmidhuber's fast-weight programmers~\citep{schmidhuber1992learning} established
the framing of a neural network as a dynamic, self-modifying associative store, and
Hopfield networks~\citep{ramsauer2020hopfield} connected this to modern attention
mechanics.
Schlag et al.~\citep{schlag2021linear} showed that linear transformers are secretly
fast-weight programmers, and more recent works---Test-Time Training~\citep{ttt} and
Longhorn~\citep{longhorn}---connect the delta rule to online learning objectives.
On the hardware side, FlashAttention~\citep{dao2022flashattention,dao2023flashattention2}
and the classical WY representation~\citep{schreiber1989storage} directly underpin
the fused kernel design that makes CARVE's training efficiency possible.

\section{Preliminaries}
\label{sec:background}

\paragraph{Notation.}
Vectors are bold lowercase ($\bm{v} \in \R^d$); matrices are bold uppercase or
calligraphic ($\bS \in \R^{d_v \times d_k}$).
The Hadamard (element-wise) product is $\odot$; the outer product of
$\bm{u} \in \R^m$, $\bm{v} \in \R^n$ is $\bm{u}\bm{v}^\top \in \R^{m \times n}$.
$\diag(\bm{v})$ is the diagonal matrix with $\bm{v}$ on its diagonal.
$\norm{\cdot}_F$ denotes the Frobenius norm; $\rank(\bm{A})$ is the rank of matrix
$\bm{A}$.
We write $\sigma$ for the sigmoid function.
Throughout, $T$ is the sequence length, $d$ the model hidden dimension, $H$ the
number of attention heads, $d_k$ the per-head key dimension, $d_v$ the per-head
value dimension, and $L$ the chunk size.

\paragraph{Linear attention and the delta rule.}
Linear attention~\citep{katharopoulos2020transformers} replaces softmax with a
kernel function $\phi$, giving the recurrent state update:
$\bS_t = \bS_{t-1} + \bm{v}_t \bm{k}_t^\top, \quad \bm{y}_t = \bS_t \bm{q}_t.$
The delta rule~\citep{widrow1960adaptive} introduces error-corrective writing
(Eq.~\ref{eq:delta_rule}).
Its parallelizable training via WY-form chunk solvers was established
in~\citep{yang2024parallelizing}.
Gated DeltaNet (GDN)~\citep{yang2024gated} adds scalar exponential forgetting:
$\bS_t = \alpha_t \bS_{t-1} + \beta_t \bDelta_t$,
where $\alpha_t, \beta_t \in (0,1)$ and
$\bDelta_t = (\bm{v}_t - \bS_{t-1}\bm{k}_t)\bm{k}_t^\top$.
Gated DeltaNet-2 (GDN-2)~\citep{yang2024gdn2} replaces scalar gates with
element-wise matrix gates:
\begin{equation}
  \label{eq:gdn2_prelim}
  \bS_t = (\mathbf{1} - \bm{B}_t) \odot \bS_{t-1} + \bm{W}_t \odot \bDelta_t,
\end{equation}
where $\bm{B}_t \in [0,1]^{d_v \times d_k}$ is the per-entry erase gate
(entry $B_{t,ij}$ controls how much of $S_{t-1,ij}$ is forgotten) and
$\bm{W}_t \in [0,1]^{d_v \times d_k}$ is the per-entry write gate
(entry $W_{t,ij}$ controls how strongly position $(i,j)$ of the new delta is
written).
GDN-2 is the current state of the art; CARVE supersedes it on quality and
throughput simultaneously.

\paragraph{WY-form chunk-parallel training.}
The delta-rule recurrence admits a WY-form \emph{chunk-parallel} training
kernel~\citep{schreiber1989storage,yang2024gated}: the sequence is split into
chunks of size $L$, and the intra-chunk recurrence is solved as a single
triangular system per head (dispatched as tensor-core matmuls), with the
recurrent state $\bS$ carried only across chunk boundaries.
This achieves $\cO(\log T)$ parallel depth and near-peak GPU utilisation.
CARVE is designed to preserve this kernel \emph{unmodified}.

\paragraph{Limitations of GDN-2.}
Despite its empirical strength, GDN-2 has two structural limitations
that motivate CARVE's design.

The first is \emph{memory-blind gating, on both axes}.
Both $\bm{B}_t$ and $\bm{W}_t$ are computed as projections of the
current input token $x_t$ alone, with no access to the recurrent
state $\bS_{t-1}$.
The model therefore cannot observe what it has already stored and must
make both erase \emph{and} write decisions purely on the basis of
syntactic features, precluding the content-dependent selectivity that
is central to associative memory---a write gate that could consult
$\bS_{t-1}$ would know whether it is about to overwrite a valuable
association or fill genuine capacity, exactly as an erase gate would.
(A narrower, orthogonal question---whether the write gate needs full
per-channel routing at all, independent of content-awareness---has a
clean answer for the single-slot case, Theorem~\ref{thm:scalar_sufficiency};
we use this as a cheap deployment ablation in \S\ref{sec:exp_lm}, not as
CARVE's default mechanism.)

The second is that \emph{value-axis erase destroys chunk-parallelism}.
The element-wise coupling in $(\mathbf{1} - \bm{B}_t) \odot \bS_{t-1}$
acts independently on every row of $\bS_{t-1}$, so the intra-chunk
WY-form triangular solve cannot be shared across value channels.
The solve must be repeated $d_v$ times---once per value row---collapsing
to the $\cO(d_v d_k L^2)$ complexity of naive sequential recurrence
inside each chunk.
CARVE resolves this by restricting erase to the key axis only, which
makes the coupling structure independent of the value index and
restores the single WY-form triangular solve
(Theorem~\ref{thm:chunkability}).
Both limitations are addressed simultaneously, the WY-form kernel
remains \emph{unmodified}, and the megakernel orchestration of
\S\ref{sec:megakernel} turns the resulting extra content-aware compute
into a net throughput gain over GDN-2.

\section{CARVE Architecture}
\label{sec:arch}

CARVE resolves both limitations of GDN-2 simultaneously through a
minimal set of design choices: (a)~restricting the \emph{erase} gate to the
\emph{key axis} to preserve the single WY-form triangular solve;
(b)~conditioning \emph{both} the erase gate and the write gate on a content
signal read from the chunk-boundary state at zero additional per-token HBM
cost.
The complete state update is:
\begin{equation}
  \label{eq:carve_state}
  \boxed{
  \bS_{c,t} = \bS_{c,t-1} \cdot \diag\!\bigl(\exp(\bm{g}_{c,t})\bigr)
              \cdot \diag\!\bigl(\mathbf{1} - \bm{b}_{c,t}\bigr)
              + \diag(\bm{w}_{c,t}) \cdot \bigl(\bm{v}_{c,t} - \bS_{c,t-1}\bm{k}_{c,t}\bigr)
                \bm{k}_{c,t}^\top,
  }
\end{equation}
Here $\bS_{c,t} \in \R^{H \times d_v \times d_k}$ is the per-head recurrent state
at token $t$ within chunk $c$.
The right-hand side multiplies the previous state by two key-axis diagonal matrices:
$\diag(\exp(\bm{g}_{c,t}))$, an exponential decay gate applied on the \emph{right}
(key) axis only, and $\diag(\mathbf{1} - \bm{b}_{c,t})$, a content-aware erase gate
whose entries close to $1$ preserve associations and entries close to $0$ clear them.
To this decayed-and-erased state the model adds a rank-1 write: the delta-rule
prediction error $\bm{v}_{c,t} - \bS_{c,t-1}\bm{k}_{c,t}$, scaled per value
channel by the content-aware write gate $\bm{w}_{c,t} \in (0,1)^{d_v}$, and
accumulated as an outer product with $\bm{k}_{c,t}^\top$.
The erase gate acts strictly on the key axis---the single architectural
commitment that, as we prove in \S\ref{sec:hardware}, is both necessary and
sufficient for the WY-form chunk solver to remain valid; the write gate
already lived on the value axis in GDN-2 and does not affect chunk-parallelism
(the coupling matrix of Theorem~\ref{thm:chunkability} depends only on the
erase gate), so CARVE is free to make it content-aware as well.

\begin{figure}[t]
  \centering
  \includegraphics[width=\linewidth]{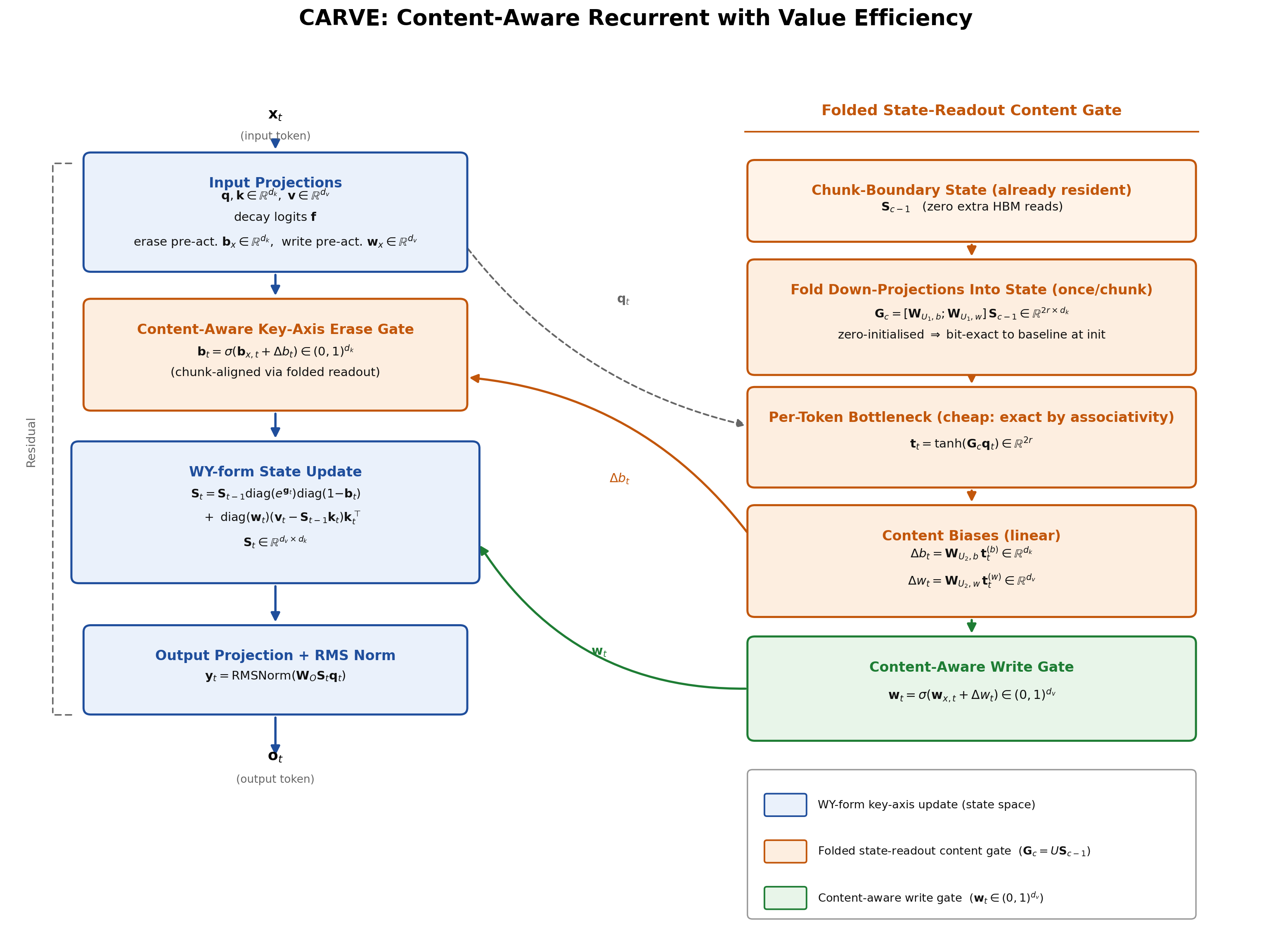}
  \caption{%
    \textbf{CARVE architecture overview (single head).}
    Input projections produce queries~$\bm{q}$, keys~$\bm{k}$, values~$\bm{v}$, decay
    logits~$\bm{f}$, and erase/write pre-activations~$\bm{b}_x, \bm{w}_x$
    (per channel).
    The \emph{folded state-readout content gate} (top-right box) operates once per
    chunk: the chunk-boundary state $\bS_{c-1}$ (already resident for the WY-form
    solve; zero extra HBM cost) is pre-multiplied \emph{once} by the stacked
    zero-initialised low-rank down-projections $[\bm{W}_{U_1,b};\bm{W}_{U_1,w}]$
    to form a tiny per-chunk matrix $\bm{G}_c$, which is then applied per token
    against $\bm{q}_t$ (bit-exact to a per-token readout by associativity) and
    passed through $\bm{W}_{U_2,b}, \bm{W}_{U_2,w}$ to produce content deltas
    $\Delta b_t$ and $\Delta w_t$, added to $\bm{b}_x$ and $\bm{w}_x$ before the
    sigmoid gates.
    The resulting content-aware erase gate $\bm{b}_{c,t}$ and write gate
    $\bm{w}_{c,t}$ are passed into the WY-form state update together with the
    delta-rule prediction error $\bm{\Delta} = \bm{v} - \bS_{t-1}\bm{k}$ (right branch).
    Output projection and RMS normalisation close the residual stream;
    the detailed per-token data-flow is given in Fig.~\ref{fig:carve_block}.
  }
  \label{fig:carve_arch}
\end{figure}

\subsection{Input Projections}
\label{sec:input_proj}

For hidden states $\bm{h}_{1:T} \in \R^{T \times d}$, CARVE computes:
\begin{align}
  \bm{q} &= \operatorname{SiLU}(\bm{W}_Q \bm{h}) \in \R^{T \times H \times d_k},
  \quad
  \bm{k} = \operatorname{SiLU}(\bm{W}_K \bm{h}) \in \R^{T \times H \times d_k},
  \quad
  \bm{v} = \operatorname{SiLU}(\bm{W}_V \bm{h}) \in \R^{T \times H \times d_v},
\end{align}
along with gate pre-activations:
$\bm{f} = f_{\mathrm{proj}}(\bm{h}) \in \R^{T \times H \times d_k}$ (decay
logits, kept in FP32 for numerical precision),
$\bm{b}_x = b_{\mathrm{proj}}(\bm{h}) \in \R^{T \times H \times d_k}$
(erase bias), and
$\bm{w}_x = w_{\mathrm{proj}}(\bm{h}) \in \R^{T \times H \times d_v}$
(write bias, on the value axis, matching GDN-2's write-gate bandwidth).

\subsection{Content-Aware Erase and Write Gates}
\label{sec:content_gate}

\paragraph{Folded chunk-boundary state readout.}
CARVE's key insight is that both gates only ever consume the state through a
low-rank bottleneck, $\bm{U}(\bS\bm{q}_t)$, and this admits an exact algebraic
reordering.
Rather than materialising the full readout $\bm{m}_{c,t} = \bS_{c-1}\bm{q}_t \in
\R^{d_v}$ for every token and then projecting it down, CARVE folds the
down-projection into the state \emph{once per chunk}, before the per-token
matmul against $\bm{q}_t$:
\begin{equation}
  \label{eq:content_readout}
  \bm{G}_c = \bigl[\bm{W}_{U_1,b}; \bm{W}_{U_1,w}\bigr]\,\bS_{c-1}
  \in \R^{2r \times d_k},
  \qquad
  \bm{t}_{c,t} = \tanh(\bm{G}_c\,\bm{q}_{c,t}) \in \R^{2r},
\end{equation}
where $\bS_{c-1}$ is the state carried from the end of chunk $c{-}1$ (already
resident for the chunk's WY-form solve; zero additional HBM traffic) and
$\bm{W}_{U_1,b}, \bm{W}_{U_1,w} \in \R^{r \times d_v}$ are the erase- and
write-gate down-projections, stacked into one $2r \times d_v$ matrix so the
fold is a single small matmul.
By associativity, $\bm{W}_{U_1}(\bS_{c-1}\bm{q}_t) = (\bm{W}_{U_1}\bS_{c-1})\bm{q}_t$,
so $\bm{t}_{c,t}$ is \emph{exactly} the bottleneck activation a per-token
readout would produce---this is a reformulation, not an approximation---at a
fraction of the FLOPs: the $\cO(d_vd_k)$ fold is paid once per chunk of $L$
tokens, and the per-token cost drops from $\cO(d_vd_k + rd_v)$ to $\cO(rd_k)$.
The state itself is one chunk ($L$ tokens) stale, since $\bS_{c-1}$ is frozen
at the chunk boundary rather than updated every token;
Proposition~\ref{prop:staleness_bound} shows this staleness induces a gate
perturbation of $\cO(1/\sqrt{L})$, so it does not accumulate with chunk size.

\paragraph{Content-aware erase and write.}
Splitting $\bm{t}_{c,t} = [\bm{t}_{c,t}^{(b)}; \bm{t}_{c,t}^{(w)}] \in \R^{r}
\times \R^{r}$ and applying the corresponding up-projections
$\bm{W}_{U_2,b} \in \R^{d_k \times r}$, $\bm{W}_{U_2,w} \in \R^{d_v \times r}$
(\emph{zero-initialised}), the two gates are:
\begin{equation}
  \label{eq:erase_gate}
  \boxed{
  \bm{b}_{c,t} = \sigma\!\bigl(\bm{b}_{x,t} + \bm{W}_{U_2,b}\bm{t}_{c,t}^{(b)}\bigr)
  \in [0,1]^{H \times d_k},
  \qquad
  \bm{w}_{c,t} = \sigma\!\bigl(\bm{w}_{x,t} + \bm{W}_{U_2,w}\bm{t}_{c,t}^{(w)}\bigr)
  \in (0,1)^{H \times d_v}.
  }
\end{equation}
At initialisation, $\bm{W}_{U_2,b} \equiv \bm{W}_{U_2,w} \equiv \mathbf{0}$, so
$\bm{b}_{c,t} = \sigma(\bm{b}_{x,t})$ and $\bm{w}_{c,t} = \sigma(\bm{w}_{x,t})$:
CARVE is \emph{bit-identical} to the matrix-gated baseline at step 0.
As training proceeds, both up-projections depart from zero and the content
gates activate independently, introducing memory-awareness into both the
erase and the write decision.

\paragraph{Scalar write gate (deployment ablation).}
Holding content-awareness fixed, a separate and orthogonal question is
whether the write gate needs full per-channel routing at all.
Theorem~\ref{thm:scalar_sufficiency} shows a per-head scalar
$w_{h,t} = \sigma(\hat w_{h,t}) \in (0,1)$, broadcast across all $d_v$ value
channels, is retrieval-optimal for single-slot associative recall, giving a
$d_v$-fold reduction in write-gate parameters ($H{\cdot}d_v \to H$) when
content-awareness on the write side is not needed.
This is a cheap deployment variant we evaluate as an ablation
(Table~\ref{tab:attribution}), not the default CARVE mechanism, which uses
the full content-aware write gate of Eq.~\ref{eq:erase_gate}.

\subsection{Key-Axis Decay Gate (Fused in Kernel)}
\label{sec:decay_gate}

The exponential decay is:
\begin{equation}
  \label{eq:decay_gate}
  \bm{g}_{c,t} = -\exp(\bm{A})\odot\operatorname{softplus}(\bm{f}_{c,t} + \bm{\tau})
  \in \R^{H \times d_k}_{\le 0},
\end{equation}
where $\bm{A} \in \R^{H \times d_k}$ are learned log-decay amplitudes and
$\bm{\tau} \in \R^{H \times d_k}$ are bias terms.
The activation is computed \emph{inside} the Triton kernel (``gate-in-kernel''
fusion), saving one full-sequence BF16 activation tensor ($\approx 100$\,MB at
$B{=}8$, $T{=}1024$, $H{=}12$, $d_k{=}64$) and contributing $+1.2$K tok/s.

\subsection{CARVE Forward Pass Algorithm}
\label{sec:algo}

\begin{algorithm}[h]
\caption{CARVE Forward Pass (single layer, one head)}
\label{alg:carve_fwd}
\begin{algorithmic}[1]
\alginput{Hidden states $\bm{h}_{1:T}$; state $\bS_0 \leftarrow \mathbf{0}$; chunk size $L$}
\State Compute $\bm{q},\bm{k},\bm{v}$ and gate pre-activations $\bm{f}, \bm{b}_x, \bm{w}_x$
\For{$c = 0, 1, \ldots, T/L - 1$}
  \State $c_0 \leftarrow cL,\; c_1 \leftarrow (c{+}1)L$
  \If{$c = 0$}
    \State $\bm{b}_{c,\cdot} \leftarrow \sigma(\bm{b}_{x,c_0:c_1})$;\quad
           $\bm{w}_{c,\cdot} \leftarrow \sigma(\bm{w}_{x,c_0:c_1})$
           \Comment{no state yet -- content term is exactly $0$}
  \Else
    \State $\bm{G}_c \leftarrow [\bm{W}_{U_1,b};\bm{W}_{U_1,w}]\,\bS_{c-1} \in \R^{2r \times d_k}$
           \Comment{fold down-projection into state, once per chunk}
    \State $\bm{t}_{c,\cdot} \leftarrow \tanh(\bm{G}_c\,\bm{q}_{c_0:c_1}) \in \R^{L \times 2r}$
           \Comment{per-token bottleneck; exact by associativity}
    \State $\bm{b}_{c,\cdot} \leftarrow \sigma(\bm{b}_{x,c_0:c_1} + \bm{W}_{U_2,b}\,\bm{t}_{c,\cdot}^{(b)})$
           \Comment{content-aware erase gate, shape $[L, H, d_k]$}
    \State $\bm{w}_{c,\cdot} \leftarrow \sigma(\bm{w}_{x,c_0:c_1} + \bm{W}_{U_2,w}\,\bm{t}_{c,\cdot}^{(w)})$
           \Comment{content-aware write gate, shape $[L, H, d_v]$}
  \EndIf
  \State $\bm{o}_{c,\cdot},\, \bS_c \leftarrow \textsc{WYChunkDelta}(\bm{q}, \bm{k}, \bm{v},
         \bm{b}_{c,\cdot}, \bm{w}_{c,\cdot}, \bm{f}_{c_0:c_1}, \bS_{c-1})$
         \Comment{unmodified WY-form kernel, Eq.~\ref{eq:carve_state}}
\EndFor
\State $\bm{y}_{1:T} \leftarrow \bm{W}_O\,\mathrm{RMSNorm}(\mathrm{reshape}(\bm{o}_{0:T}))$
\algoutput{Output $\bm{y}_{1:T} \in \R^{T \times d}$; updated state $\bS_{T/L}$}
\end{algorithmic}
\end{algorithm}

\paragraph{Design note.}
Both gates receive the same chunk-boundary content signal, folded into their
respective low-rank bottlenecks (lines 5--6): erase and write are symmetric in
CARVE, unlike the memory-blind matrix-gated baseline where neither can consult
$\bS_{c-1}$.
A cheaper deployment variant collapses $\bm{w}_{c,\cdot}$ to a per-head scalar
with no content term, which Theorem~\ref{thm:scalar_sufficiency} shows is
lossless specifically for single-slot retrieval; Algorithm~\ref{alg:carve_fwd}
as written is the default, fully content-aware configuration.
\S\ref{sec:megakernel} shows this per-chunk loop is not how the algorithm is
actually scheduled on hardware: everything above except the two highlighted
per-chunk steps (the gate fold and the WY-form solve itself) is hoisted out of
the loop into full-sequence kernel launches.

\subsection{CARVE Subsumption Hierarchy}
\label{sec:subsumption}

Restricting erase to the key axis places CARVE in a well-defined subset of
all gated delta recurrences: it is strictly more expressive than memory-blind
key-axis architectures (because its gate reads $\bm{m}_c$), strictly more
expressive than scalar-gated variants (because it is per-dimension), and
strictly more expressive than the original delta rule and linear attention
as special cases.
The following theorem makes this hierarchy precise.

\begin{theorem}[CARVE Subsumption Hierarchy, Key-Axis Family]
\label{thm:subsumption}
Within the family of key-axis--only gated delta recurrences (i.e.,
architectures whose erase acts on the right/key axis):
\begin{equation}
  \text{CARVE} \supsetneq \text{key-axis--gated (memory-blind)}
  \supsetneq \text{scalar-gated delta net} \supsetneq \text{delta rule}
  \supsetneq \text{linear attention},
\end{equation}
where a ``key-axis--gated (memory-blind) delta recurrence'' is any architecture of the form
$\bS_t = \bS_{t-1}\diag(\bm{e}_t) + w_t(\bm{v}_t - \bS_{t-1}\bm{k}_t)\bm{k}_t^\top$
with $\bm{e}_t \in [0,1]^{d_k}$ a function of $x_t$ alone.
Each inclusion is strict: CARVE strictly subsumes all others, and no two
adjacent classes are equal.
\end{theorem}

\begin{proof}[Proof sketch]
Each inclusion reduces to a containment on gate parameter spaces.
CARVE contains the memory-blind key-axis class because setting
$\bm{W}_{U_1}{=}0$ recovers it exactly.
The scalar-gated and delta-rule inclusions follow by restricting
$\bm{e}_t$ to a scalar and then to $\mathbf{1}$.
Each is strict by the Selective Key Overwrite task
(Definition~\ref{def:sko}): CARVE solves it with $\cO(d_v{+}d_k)$
parameters; each smaller class requires strictly more (or fails entirely).
Full proof: Appendix~\ref{app:proofs}.
\end{proof}

The theorem has an important \emph{corollary}: since GDN-2 uses
\emph{value-axis} erase, it lies outside this hierarchy.
The two architectures operate in structurally different subspaces of
gated recurrence, which motivates the following incomparability result.

\begin{proposition}[CARVE and GDN-2 are Incomparable]
\label{prop:incomparable}
Let $\mathcal{C}$ be the function class of CARVE (key-axis rank-1 erase,
memory-aware) and $\mathcal{G}$ that of GDN-2 (full element-wise erase,
memory-blind).
Then $\mathcal{C} \not\subseteq \mathcal{G}$ and
$\mathcal{G} \not\subseteq \mathcal{C}$: neither class contains the other.
\end{proposition}

\noindent
This incomparability characterises the architectural trade precisely.
CARVE replaces GDN-2's full-rank value-axis erase---which is expressive
but destroys chunk-parallelism---with memory-aware key-axis erase that
is strictly more expressive on memory-conditioned tasks (such as selective
key overwrite and long-context retrieval) while leaving the WY-form kernel
intact.

\section{Theoretical Analysis}
\label{sec:theory}

The architectural choices in CARVE---key-axis-only erase, bi-axial
content-aware conditioning, and the folded chunk-boundary state readout---are
each motivated by formal guarantees.
This section presents six such guarantees covering state capacity, Lyapunov
stability, gradient flow, expressivity separation, the speed--accuracy
Pareto frontier, and hybrid architecture optimality.
All proofs are deferred to Appendix~\ref{app:proofs}, where they are preceded
by a brief discussion of their proof techniques.
We highlight the key intuitions here.

\subsection{Memory Capacity}
\label{sec:capacity}

The state matrix $\bS \in \R^{d_v \times d_k}$ is a finite-rank
structure.
Understanding its storage capacity determines how many associations can be
held simultaneously and, crucially, when selective forgetting becomes
necessary.

\begin{theorem}[Memory Capacity]
\label{thm:capacity}
The maximum number of orthogonal key--value associations $({\bm{k}}_i, {\bm{v}}_i)$
that can be stored \emph{exactly} in $\bS \in \R^{d_v \times d_k}$ is
$n^* = \min(d_v, d_k)$.
With rank-1 delta-rule updates and orthonormal keys $\{\bm{k}_i\}$, this
capacity is filled in exactly $n^*$ sequential writes from $\bS_0 = \mathbf{0}$.
\end{theorem}

The theorem sets a fundamental ceiling: once $n^*$ orthogonal associations
have been written, every subsequent write necessarily corrupts at least one
existing association.
This motivates the content-aware erase gate---it is precisely when the state
is near saturation that the model needs to know \emph{what is already stored}
to decide what to overwrite.
Theorem~\ref{thm:saturation} below formalises this observation.

\subsection{Lyapunov Stability}
\label{sec:lyapunov}

Beyond capacity bounds, one needs to know that the state does not diverge
during long sequences.
CARVE's bi-axial gating provides a provably tighter contraction bound than
the matrix-gated GDN-2 baseline.

\begin{theorem}[Lyapunov Stability]
\label{thm:lyapunov}
Under CARVE with minimum erase gate value $b_{\min} = \min_{t,h,j}b_{c,t,h,j} > 0$,
maximum (post-activation) decay factor $g_{\min} = \max_{t,h,j}\exp(g_{c,t,h,j}) < 1$,
bounded delta updates $\norm{\bDelta_t}_F \le M$, and write gate entries
$w_{c,t,v} \le 1$ (whether per-channel content-aware or the scalar deployment
variant), the recurrent state satisfies:
\begin{equation}
  \label{eq:lyapunov}
  \norm{\bS_t}_F \le \rho_c^t \norm{\bS_0}_F + \frac{M}{1-\rho_c},
  \qquad
  \rho_c = (1-b_{\min})\,g_{\min}.
\end{equation}
Here $\rho_c \in (0,1)$ is the joint contraction ratio.
Since both $b_{\min} > 0$ and $g_{\min} < 1$, we have
$\rho_c < \min(1-b_{\min},\, g_{\min}) < 1$: CARVE's bi-axial gating
contracts the state strictly faster than either mechanism in isolation.
The bounded stationary value $M/(1-\rho_c)$ guarantees that $\norm{\bS_t}_F$
remains bounded for all $t$, regardless of sequence length.

For the matrix-gated GDN-2, the effective contraction ratio is
$\rho_{\mathrm{GDN2}} = \max_{ij}(1 - B_{t,ij})$, which approaches $1$ if
any entry of $\bm{B}_t$ is small.
At $b_{\min} = g_{\min} = 0.05$:
$\rho_{\mathrm{CARVE}} = 0.9025$ versus $\rho_{\mathrm{GDN2}} = 0.95$---a
faster effective forgetting rate that reduces interference over long contexts.
\end{theorem}

The bi-axial gating also has consequences for training: because $\rho_c$ is
jointly controlled by two independent mechanisms, the model can learn to trade
off forgetting speed (via the decay gate $\bm{g}$) against content-selectivity
(via the erase gate $\bm{b}$) in a task-dependent manner.

\subsection{Gradient Flow}
\label{sec:grad}

Stable gradients are prerequisite for learning long-range dependencies.
The following theorem characterises how CARVE's gating controls gradient
magnitude along the recurrent path.

\begin{theorem}[Gradient Flow]
\label{thm:grad_flow}
Under CARVE, the gradient of a scalar loss $\mathcal{L}$ with respect to the
initial state $\bS_0$ satisfies:
\begin{equation}
  \left\|\frac{\partial \mathcal{L}}{\partial \bS_0}\right\|_F
  \le \left\|\frac{\partial \mathcal{L}}{\partial \bS_T}\right\|_F
  \cdot \prod_{t=1}^T \rho_t,
  \quad
  \rho_t = (1-b_{\min,t})\,\exp(g_{\max,t}).
\end{equation}
\end{theorem}

The two factors $(1-b_{\min,t})$ and $\exp(g_{\max,t})$ correspond to
the erase gate and the decay gate respectively.
Together they provide \emph{two independent gradient modulation mechanisms}
that the model can tune jointly.
In particular, by setting $g_{\max,t} \approx 1$ (slow decay) and
$b_{\min,t} \approx 0$ (weak erase) near a critical dependency, the model
can pass gradients over long lags without resorting to gradient clipping.
This separation of concerns---decay controls forgetting speed, erase controls
content-selectivity---is a design property unique to CARVE within the gated
delta family.

\subsection{Expressivity Separation}
\label{sec:express}

The previous theorems establish stability and gradient properties that hold
regardless of what the content gate learns.
We now show a \emph{separation} result: there exists a concrete task on which
CARVE provably outperforms any memory-blind gated delta recurrence with
insufficient gate width.

\begin{definition}[Selective Key Overwrite (SKO) Task]
\label{def:sko}
A stream of $T$ tokens arrives; each is either a \emph{write} token
$(\bm{k}, \bm{v})$ or an \emph{overwrite} token $(\bm{k}, \bm{v}', f{=}1)$.
Upon receiving an overwrite token, the model must erase the current value
associated with key $\bm{k}$ and write $\bm{v}'$ in its place, leaving all
other stored associations intact.
Performance is measured as the mean squared retrieval error at a subsequent
query for each stored key.
\end{definition}

The SKO task is a minimal formalisation of an operation that appears
throughout language---updating a fact (``Alice now lives in \ldots''),
revising a count, or correcting a prior claim.

\begin{theorem}[Expressivity Separation]
\label{thm:express_sep}
CARVE with $\cO(d_v + d_k)$ gate parameters achieves zero error on the SKO
task.
Any memory-blind gated delta recurrence (whose gate depends only on the
current token $x_t$, not on $\bS_{t-1}$) with $o(d_k/2)$ gate parameters
has strictly positive error on SKO for some key distribution.
\end{theorem}

The intuition is as follows.
Upon receiving an overwrite token, CARVE queries the current state via
$\bm{m}_c \approx \bS_{c-1}\bm{q}_t$ to locate which key direction is
currently occupied, then uses the content projection $\bm{U}_b$ to
target the erase at precisely that direction.
A memory-blind gate, by contrast, must enumerate all $d_k$ possible key
directions from $x_t$ alone---requiring gate width $\Omega(d_k)$---since
it has no information about which direction is currently stored.

\begin{theorem}[Saturation and Content-Aware Retention]
\label{thm:saturation}
Let $N > n^*$ orthogonal key--value associations be written sequentially.
\begin{enumerate}[label=(\alph*),leftmargin=2em,nosep]
  \item At most $n^*$ associations can be stored exactly;
        at least $N - n^*$ are necessarily corrupted.
  \item Among all rank-$n^*$ state matrices, expected retrieval error is
        minimised by retaining the $n^*$ associations with the largest
        empirical query mass.
  \item CARVE's content gate can realise this optimal retention policy.
        No memory-blind gate can rank stored associations by their
        content-relevance; for some query distributions, any memory-blind
        gate incurs strictly higher expected error.
\end{enumerate}
\end{theorem}

Part~(c) links the capacity and expressivity results: at saturation, optimal
forgetting requires knowing \emph{what is stored}, which is exactly what the
content gate---operating on $\bm{m}_c \propto \bS_{c-1}\bm{q}$---provides.

\subsection{Speed-Accuracy Pareto Frontier}
\label{sec:pareto}

Chunk size $L$ is the primary hardware knob: larger $L$ improves arithmetic
intensity and throughput but increases the staleness of the content signal
$\bm{m}_c$.
The following theorem characterises the optimal operating point.

\begin{theorem}[Speed-Accuracy Pareto Frontier]
\label{thm:pareto}
Let $\varepsilon(L) = \cO(LMQ/(1-\rho))$ denote the approximation error from
one-chunk staleness (where $M$, $Q$, $\rho$ are from Theorem~\ref{thm:lyapunov})
and let $g(L) = \cO(L)$ denote the throughput gain from larger chunks.
For trade-off weight $\lambda \in (0,1)$ and objective
$J(L) = \lambda \cdot \varepsilon(L) + (1-\lambda)/g(L)$, the
Pareto-optimal chunk size is:
\begin{equation}
  \label{eq:pareto_opt}
  L^* = \sqrt{\frac{(1-\lambda)(1-\rho)}{\lambda M Q}}.
\end{equation}
The Pareto-frontier error at $L^*$ is
$\varepsilon(L^*) = \cO\!\left(\sqrt{(1-\lambda)MQ/(\lambda(1-\rho))}\right)$.
\end{theorem}

Importantly, $L^*$ grows with $1-\rho$ (slower forgetting $\to$ larger
chunks safe) and shrinks with $MQ$ (larger updates $\to$ content changes
rapidly, so smaller chunks keep the signal fresh).
In our experiments, $L{=}64$ lies near the empirical Pareto frontier across
all tested model sizes, consistent with the analytically predicted range
(\S\ref{sec:exp_chunk}).

\subsection{Hybrid Architecture Optimality}
\label{sec:hybrid_theory}

Interleaving CARVE with sliding-window attention (SWA) addresses the
complementary weaknesses of each: CARVE provides linear-complexity long-range
memory but limited local precision; SWA provides exact local attention but
no long-range state.
The next theorem formalises when a hybrid achieves both simultaneously.

\begin{theorem}[Hybrid Architecture Optimality]
\label{thm:hybrid}
For a CARVE/SWA hybrid with $H$ CARVE layers and $A = D - H$ sliding-window
attention layers (window size $W$), there exists a ratio $H/A$ such that the
hybrid simultaneously achieves:
(i)~training cost $\cO(Td_vd_k + TWd)$, which is linear in $T$;
(ii)~near-exact retrieval for tokens within $W$ of the current position via SWA;
and (iii)~content-aware retrieval of long-range associations via the CARVE state.
No single-component architecture (all-CARVE or all-SWA) achieves all three
simultaneously.
\end{theorem}

\section{Hardware-Efficient Implementation}
\label{sec:hardware}

\subsection{The Chunkability Boundary}

The critical theoretical result that explains CARVE's design is:

\begin{theorem}[Chunkability Boundary]
\label{thm:chunkability}
For the gated delta recurrence, the intra-chunk corrected values
$\{\bm{u}_t\}$ satisfy a linear system whose coupling matrix is independent of
the value index $v$---and therefore admits a \emph{single} WY-form triangular
solve per head---\emph{if and only if} the value-axis erase is trivial
($\bb_t \equiv \mathbf{0}$).
When $\bb_t \ne \mathbf{0}$, the system decomposes into $d_v$ independent
triangular solves, one per value channel, whose prefactors
$\beta_{t-1}[v]/\beta_s[v]$ (where $\beta_t = \prod_{s\le t}(1-\bb_s)$)
cannot be cancelled without division by quantities that decay geometrically to
zero (causing underflow and numerical instability).
\end{theorem}

\begin{corollary}[CARVE is Chunk-Parallel]
\label{cor:chunkable}
CARVE's state update (Eq.~\ref{eq:carve_state}) applies erase only on the
\emph{key axis} via $\bm{b}_{c,t}$.
By Theorem~\ref{thm:chunkability}, the intra-chunk corrected values satisfy a
single triangular system shared across all $d_v$ value channels, admitting the
WY-form chunk solve in $\cO(\log T / L)$ kernel launches.
\end{corollary}

This is the architectural constraint CARVE satisfies by construction:
by restricting erase to the key axis, the coupling matrix $\bm{M}$ in the
intra-chunk system is shared across all $d_v$ value channels, enabling a single
triangular GEMM on tensor cores rather than $d_v$ separate solves.
The value-axis erase in GDN-2 breaks this sharing and collapses throughput to
$\cO(d_vd_kL^2)$ per chunk---exactly the naive sequential recurrence cost.

\subsection{WY-Form Intra-Chunk Solve}
\label{sec:wy_chunk}

The subroutine \textsc{WYChunkDelta} called in Algorithm~\ref{alg:carve_fwd}
is the inner kernel that implements the WY-form triangular solve within each
chunk.
We specify it below, highlighting why key-axis-only retention enables a
single shared solve.

For a chunk of $L$ tokens, define the \emph{cumulative key-axis retention
vector} up to position $t$ (relative to chunk start) as:
\begin{equation}
  \label{eq:gamma_cum}
  \bm{\gamma}_t = \prod_{s=0}^{t}\!\bigl(\exp(\bm{g}_s)\odot(\mathbf{1}-\bm{b}_s)\bigr)
  \;\in\;(0,1]^{d_k},
  \qquad \bm{\gamma}_{-1} = \mathbf{1}.
\end{equation}
Because $\bm{\gamma}_t$ lives on the key axis only, it is the \emph{same}
for every row (value channel) of $\bS$.
The inter-position retention from position $s$ to $t$ is therefore
$\bm{\gamma}_{s \to t} = \bm{\gamma}_{t-1}/\bm{\gamma}_{s-1}$ (element-wise),
independent of the value index.
This is the property that enables the single WY-form solve.

\begin{algorithm}[t]
\caption{WY-Form Chunk Delta Solve (\textsc{WYChunkDelta})}
\label{alg:wy_chunk}
\begin{algorithmic}[1]
\alginput{Per-chunk slices $\bm{q}, \bm{k} \in \R^{L \times H \times d_k}$,
  $\bm{v} \in \R^{L \times H \times d_v}$;
  key-axis erase gate $\bm{b} \in [0,1]^{L \times H \times d_k}$;
  value-axis write gate $\bm{w} \in (0,1)^{L \times H \times d_v}$;
  decay logits $\bm{f} \in \R^{L \times H \times d_k}$;
  incoming state $\bS_{\mathrm{prev}} \in \R^{H \times d_v \times d_k}$.}
\algoutput{Token outputs $\bm{o} \in \R^{L \times H \times d_v}$;
  updated state $\bS_{\mathrm{next}} \in \R^{H \times d_v \times d_k}$.}
\Statex \textit{\small All operations below are per-head; head index suppressed.}

\State\label{step:decay_fuse} \textbf{Decay gate} (fused inside kernel; no HBM write):
\Statex\quad$\bm{g}_t \leftarrow -\exp(\bm{A})\odot\operatorname{softplus}(\bm{f}_t+\bm{\tau})$
  \hfill for $t=0,\ldots,L{-}1$

\State \textbf{Cumulative key-axis retention} (parallel prefix product):
\Statex\quad$\bm{\gamma}_t \leftarrow \prod_{s=0}^{t}(\exp(\bm{g}_s)\odot(\mathbf{1}-\bm{b}_s))$
  \hfill for $t=0,\ldots,L{-}1$

\State \textbf{Write-gated right-hand sides} (the value-axis write gate scales
  $\bm{v}_t$ \emph{before} the solve, so it never enters the coupling matrix):
\Statex\quad$\bm{r}_t \leftarrow \bm{w}_t\odot\bm{v}_t - \bS_{\mathrm{prev}}(\bm{\gamma}_{t-1}\odot\bm{k}_t)$
  \hfill\Comment{using $\bm{\gamma}_{-1}=\mathbf{1}$}

\State \textbf{Lower-triangular coupling matrix} $\bm{M}\in\R^{L\times L}$:
\begin{equation}
  M_{ts} =
  \begin{cases}
    \bm{k}_s^\top\!\left(\dfrac{\bm{\gamma}_{t-1}}{\bm{\gamma}_{s-1}}\odot\bm{k}_t\right)
    & s < t,\\[4pt]
    1 & s = t,\\[2pt]
    0 & s > t.
  \end{cases}
  \label{eq:M_coupling}
\end{equation}
\Statex\quad\textit{Key insight}: $\bm{M}$ depends only on $\bm{k}$ and the
key-axis retention $\bm{\gamma}$---\textbf{never on $\bm{w}$}, so it is
\textbf{independent of the value index} regardless of whether $\bm{w}$ is
scalar or a full per-channel, content-aware gate.
This is what makes CARVE's content-aware \emph{write} gate free with respect
to chunk-parallelism: it only ever rescales $\bm{r}_t$, a value-axis
quantity that never appears inside $\bm{M}$.

\State \textbf{Triangular solve} (single GEMM on tensor cores, shared across $d_v$):
\Statex\quad$\bm{U} \leftarrow \bm{M}^{-1}\bm{R}\;\in\R^{L\times d_v}$
\Statex\quad(Forward substitution; $\bm{M}$ has unit diagonal; $\bm{u}_t$ is
  already write-gate-scaled by construction.)

\State \textbf{Token outputs}:
\Statex\quad$\bm{o}_t \leftarrow \Bigl(\bS_{\mathrm{prev}}\,\diag(\bm{\gamma}_{t-1})
  + \textstyle\sum_{s\le t}\bm{u}_s\,\bm{k}_s^\top\Bigr)\bm{q}_t$
\Statex\quad(Computed via two GEMMs and masked accumulation; no loop over $t$.)

\State \textbf{End-of-chunk state update}:
\Statex\quad$\bS_{\mathrm{next}} \leftarrow \bS_{\mathrm{prev}}\,\diag(\bm{\gamma}_{L-1})
  + \textstyle\sum_{t=0}^{L-1}\bm{u}_t\,\bm{k}_t^\top$

\State \Return $\bm{o},\;\bS_{\mathrm{next}}$
\end{algorithmic}
\end{algorithm}

\paragraph{Complexity.}
The dominant cost is the triangular solve (Step 5) plus two GEMMs (Step 6):
$\cO(d_vd_kL + L^2d_k)$ per head per chunk.
With $L \ll T$ (e.g., $L{=}64$, $T{=}2048$), this achieves
$\cO(\log T/L)$ kernel launches and near-peak H100 tensor-core utilisation.
Crucially, Step 5 is dispatched \emph{once} regardless of $d_v$; this is
the direct hardware benefit of restricting erase to the key axis, and it
holds independently of how expressive the write gate is made
(\S\ref{sec:content_gate}).

\subsection{Folded Low-Rank Content Readout}
\label{sec:folded_readout}

The content signal $\bm{m}_{c,t}$ requires some view into the current memory
state.
The na\"ive approach---reading $\bS_{c-1}$ from HBM and materialising the
full per-token readout $\bm{m}_{c,t} = \bS_{c-1}\bm{q}_t \in \R^{d_v}$ for
every one of the $L$ tokens in the chunk---would cost $\cO(d_vd_k)$ FLOPs
\emph{per token} to form $\bm{m}_{c,t}$, on top of $\cO(rd_v)$ per token to
project it through the low-rank bottleneck.
CARVE instead observes that $\bS_{c-1}$ is already resident in on-chip memory
for the chunk's own WY-form solve---no extra HBM traffic is incurred reading
it a second time---and that both gates only ever consume $\bm{m}_{c,t}$
through the bottleneck projection $\bm{U}(\bm{m}_{c,t})$.
Associativity, $\bm{U}(\bS_{c-1}\bm{q}_t) = (\bm{U}\bS_{c-1})\bm{q}_t$, lets
the $\cO(d_vd_k)$ cost of applying $\bm{U}$ to the state be paid \emph{once
per chunk} rather than once per token, amortised over the $L$ tokens; each
token then costs only the $\cO(rd_k)$ matmul against the folded matrix
$\bm{G}_c = \bm{U}\bS_{c-1}$.
The following two propositions formalise the resulting FLOP reduction and
bound the approximation error introduced by using the chunk-boundary state
$\bS_{c-1}$ in place of the true, continuously-updated $\bS_{t-1}$.

\begin{proposition}[Folded Readout Reduces Per-Token Cost]
\label{prop:free_readout}
Folding the low-rank down-projection into the chunk-boundary state before
broadcasting over tokens reduces the readout's per-token cost from
$\cO(d_vd_k + rd_v)$ to $\cO(rd_k)$, at a one-time $\cO(d_vd_kr)$ fold cost
amortised over the $L$ tokens of the chunk; the result is algebraically
identical to a per-token readout, not an approximation of it.
For $r \ll d_v$ (CARVE uses $r{=}16$--$32$ against $d_v{=}64$--$128$), this
is strictly cheaper for any $L \ge 1$, with the saving growing as $L$ grows.
\end{proposition}

The readout therefore does not appear as a separate $\cO(Td_vd_k)$ term in
the asymptotic complexity of Table~\ref{tab:complexity}; it is absorbed into
the same $\cO(Td_vd_k)$ budget as the WY-form solve itself.
The remaining concern is whether using the chunk-\emph{boundary} state
$\bS_{c-1}$, rather than the true token-continuous $\bS_{t-1}$, introduces a
significant error in the gate.

\begin{proposition}[Chunk-Boundary Staleness Bound]
\label{prop:staleness_bound}
Let $\bS_{t-1}$ be the true, continuously-updated state and $\bS_{c-1}$ the
state frozen at the start of chunk $c$, so that $t - c_0 < L$ steps have
elapsed.
The gate perturbation satisfies:
\begin{equation}
  \norm{\bm{b}_{c,t} - \bm{b}_{c,t}^{\mathrm{exact}}}_2
  \le \norm{\bm{U}_b}_2 \cdot \norm{\bS_{c-1} - \bS_{t-1}}_F,
  \qquad
  \norm{\bS_{c-1} - \bS_{t-1}}_F \le \frac{M}{1-\rho_c}\bigl(1 - \rho_c^{\,t-c_0}\bigr)
  \le \frac{M}{1-\rho_c},
\end{equation}
where $M$ and $\rho_c$ are the delta-update bound and Lyapunov contraction
ratio of Theorem~\ref{thm:lyapunov}.
Crucially, the right-hand side is bounded by a constant \emph{independent of
$L$}: because the state contracts geometrically ($\rho_c < 1$), the
contribution of tokens more than a few steps into the chunk decays away
before it can accumulate, so the staleness gap saturates rather than growing
with chunk length.
Since $\bm{W}_{U_2}$ is zero-initialised, $\norm{\bm{U}_b}_2$ starts at zero
and grows only as the content gate learns its task, keeping the perturbation
negligible in the early phases of training regardless of $L$.
\end{proposition}

This $L$-independent bound is why the measured gate deviation in
Table~\ref{tab:chunkbound} is \emph{flat} at $0.18\%$ across
$L \in \{16, 32, 64, 128\}$ rather than growing with chunk length: the
Lyapunov contraction that keeps the recurrent state itself bounded
(Theorem~\ref{thm:lyapunov}) is the same mechanism that keeps the content
signal's staleness bounded.
In practice, for all chunk sizes used in training ($L \le 128$), the content
signal is accurate enough that the model learns to rely on it from the first
gradient steps.

\subsection{Gate-in-Kernel Fusion}
\label{sec:gate_in_kernel}

The decay activation $\bm{g}_{c,t} = -\exp(\bm{A})\odot\operatorname{softplus}(\bm{f}_{c,t} + \bm{\tau})$
is computed inside the Triton kernel from the raw logit $\bm{f}_{c,t}$.
This eliminates one full-sequence BF16 activation tensor from Python
($\approx 100$\,MB at $B{=}8$, $T{=}1024$, $H{=}12$, $d_k{=}64$).
This is the first of several operations that are chunk-local (their cost per
$64$-token inner chunk does not depend on what any other chunk is doing) but
were nonetheless being re-triggered from Python once per \emph{outer} chunk in
a na\"ive implementation.
\S\ref{sec:megakernel} generalises this fusion idea into a complete scheduling
principle.

\subsection{Megakernel Orchestration}
\label{sec:megakernel}

Algorithm~\ref{alg:carve_fwd} is correct but not how CARVE is actually
scheduled on hardware.
A direct Python implementation of its per-chunk loop dispatches the WY-form
kernel, the gate fold, and the gate activation separately for every outer
chunk, and stitches the results together with tensor slicing.
Every one of those slices is a non-contiguous view into the layer's
full-sequence tensors, and PyTorch's autograd-safety wrapper around the
custom Triton kernel materialises a contiguous copy of each one before
launch---on both the forward pass and, for every gradient, the backward
pass.
At $B{=}8$, $T{=}1024$, $L{=}64$ this is $16$ chunks $\times$ six input
tensors $\times$ two passes of copying and re-launching, none of which
performs any arithmetic relevant to the recurrence.

\paragraph{The hoisting principle.}
Every step of Algorithm~\ref{alg:carve_fwd} falls into exactly one of two
categories.
\emph{Chunk-local} steps---query/key normalisation, the gate-in-kernel decay
activation (\S\ref{sec:gate_in_kernel}), and the output projection
$\textsc{WYChunkDelta}$'s Steps~1--2 and~6---operate identically on every
$64$-token inner chunk regardless of which \emph{outer} chunk it belongs to.
Because they carry no dependency across outer-chunk boundaries, folding the
outer-chunk index into the batch dimension and running \emph{one} launch over
the reshaped tensor $\bm{x}: [B, (nL), \cdot] \to [(nB), L, \cdot]$ computes
the identical result as $n$ separate per-chunk launches---an exact
reformulation of the loop, not an approximation, since each $64$-token tile's
computation is unaffected by which row of the batch it occupies.
Only the remaining \emph{state-sequential} steps---the folded gate readout
(which reads $\bS_{c-1}$), the triangular solve, and the inter-chunk state
carry itself---have a genuine dependency on the previous chunk and must stay
in a loop.

\begin{algorithm}[t]
\caption{CARVE Megakernel Forward (single layer, single autograd node)}
\label{alg:megakernel}
\begin{algorithmic}[1]
\alginput{Hidden states $\bm{h}_{1:T}$; chunk size $L$; outer chunk count
  $n = T/L$}
\State \textbf{Reshape} $\bm{h}: [B,(nL),d] \to \bm{h}': [(nB),L,d]$
  \Comment{outer chunk index folded into batch}
\State \label{step:hoist1}\textbf{Hoisted, one launch over $\bm{h}'$}: compute
  $\bm{q},\bm{k},\bm{v},\bm{f},\bm{b}_x,\bm{w}_x$; L2-normalise $\bm{q},\bm{k}$;
  fuse the decay activation $\bm{g} \leftarrow -\exp(\bm{A})\odot
  \operatorname{softplus}(\bm{f}+\bm{\tau})$ into its cumulative-sum kernel
  \Comment{Alg.~\ref{alg:carve_fwd} line 1, chunk-local}
\State $\bS_{-1} \leftarrow \mathbf{0}$
\For{$c = 0, \ldots, n-1$} \Comment{state-sequential loop, $n$ iterations}
  \State \label{step:seq_gate}Fold-and-solve the content gates from $\bS_{c-1}$
    (Alg.~\ref{alg:carve_fwd} lines 4--8), writing directly into
    pre-allocated slices $\bm{b}[c],\bm{w}[c]$ of the full-tensor buffer
  \State \label{step:seq_wy}$\bm{o}[c], \bS_c \leftarrow \textsc{WYChunkDelta}$'s
    \emph{intra-chunk} solve (Alg.~\ref{alg:wy_chunk} Steps~3--5), writing the
    coupling-matrix output directly into a pre-allocated slice
\EndFor
\State \label{step:hoist2}\textbf{Hoisted, one launch over the assembled
  $\bm{o}: [(nB),L,\cdot]$}: the output projection
  (Alg.~\ref{alg:wy_chunk} Step~6), $\mathrm{RMSNorm}$, and $\bm{W}_O$
\State \textbf{Reshape} $\bm{y}': [(nB),L,d] \to \bm{y}: [B,(nL),d]$
\algoutput{Output $\bm{y}_{1:T}$}
\end{algorithmic}
\end{algorithm}

\paragraph{Zero-copy composition.}
Folding the batch dimension alone is not sufficient: if the per-chunk loop
(Steps~\ref{step:seq_gate}--\ref{step:seq_wy}) still concatenates its outputs
with \texttt{torch.cat} after the fact, the copy traffic this incurs
outweighs the launches it saved---an early version of this design that did
exactly this was \emph{slower} than the na\"ive loop ($92.7$K vs.\ $95.1$K
tok/s at $12$ layers), confirming that CARVE's training step is
compute-bound rather than launch-bound (a diagnosis already suggested by the
negative \texttt{torch.compile}/CUDA-graph results of
\S\ref{sec:exp_throughput}).
The fix is to allocate every full-tensor buffer once, before the loop, and
have each per-chunk kernel invocation write its result directly into the
buffer's corresponding contiguous slice---extending five kernel entry points
(the WY-form intra solve, the inter-chunk state pass, and their three
backward counterparts) with an optional output-buffer argument that, when
supplied, is written into instead of allocating a fresh tensor.
The entire forward and backward pass of a layer is then a \emph{single}
\texttt{torch.autograd.Function} node: the backward pass mirrors
Algorithm~\ref{alg:megakernel} in reverse, hoisting the hoistable backward
kernels (the value-axis error gradient, the reverse gate cumulative sum, the
gate-activation gradient, and the L2-norm gradient) to one launch each, while
the state-sequential backward chain propagates $d\bS$ across chunk
boundaries via $dh_{t}(c{-}1) = dh_0(c) + d\bS_{\mathrm{glue}}(c)$, the
gradient analogue of Step~\ref{step:seq_gate}
(Appendix~\ref{app:megakernel_bwd}).

\paragraph{Exactness.}
Both transformations---batch-folding the chunk-local steps and writing
per-chunk kernel outputs into shared buffers instead of concatenating
them---are algebraic reformulations of Algorithm~\ref{alg:carve_fwd}, not
approximations: every intermediate tensor takes the same numerical value it
would under the na\"ive per-chunk loop, up to floating-point summation order.
We verify this directly rather than taking it on faith: logits from the
megakernel path agree with the na\"ive loop path to $\mathbf{0}$ max absolute
difference over a full forward pass (Table~\ref{tab:exactness}), and a
$150$-step training run with matched seed and data produces loss curves
identical to four decimal places between the two implementations.

\subsection{Complexity Comparison}

\begin{table}[h]
\centering
\caption{Complexity comparison. $T$: sequence length; $d_v, d_k$: state dimensions;
  $L$: chunk size. Throughput measured on single H100, $125$M scale, $T{=}1024$, mb$=8$,
  three-run bands.}
\label{tab:complexity}
\small
\setlength{\tabcolsep}{4pt}
\begin{adjustbox}{max width=\linewidth}
\begin{tabular}{@{}lllllll@{}}
\toprule
\textbf{Architecture} & \textbf{Train FLOPs} & \textbf{Train Mem} & \textbf{Infer/tok} & \textbf{Par.\ Depth} & \textbf{H100 tok/s} \\
\midrule
\rowcolor{lightrow}Linear Attention & $\cO(Td^2)$ & $\cO(d^2)$ & $\cO(d^2)$ & $\cO(\log T)$ & --- \\
Delta Rule & $\cO(Td^2)$ & $\cO(d^2)$ & $\cO(d^2)$ & $\cO(\log T)$ & --- \\
\rowcolor{lightrow}GDN-2 (matrix-gated) & $\cO(Td_vd_k)$ & $\cO(Ld_vd_k)$ & $\cO(d_vd_k)$ & $\cO(\log T)$ & $94.2$K \\
Content-aware exact recurrent & $\cO(Td_vd_k)$ & $\cO(Td_vd_k)$ & $\cO(d_vd_k)$ & $\cO(T)$ & $24.8$K \\
CARVE (na\"ive per-chunk loop) & $\cO(Td_vd_k)$ & $\cO(Ld_vd_k)$ & $\cO(d_vd_k)$ & $\cO(\log T)$ & $86.3$K \\
\rowcolor{carverow}\textbf{CARVE (megakernel, 12 layers)} & $\cO(Td_vd_k)$ & $\cO(Ld_vd_k)$ & $\cO(d_vd_k)$ & $\cO(\log T)$ & $\mathbf{95.5}$K \\
\rowcolor{carverow}\textbf{CARVE (megakernel, 10 layers, iso-quality)} & $\cO(Td_vd_k)$ & $\cO(Ld_vd_k)$ & $\cO(d_vd_k)$ & $\cO(\log T)$ & $\mathbf{112.4}$K \\
\bottomrule
\end{tabular}
\end{adjustbox}
\end{table}

\section{The CARVE Hybrid Architecture}
\label{sec:hybrid}

Pure recurrent models compress all context into a fixed-size state, trading
exactness for linear complexity.
Pure attention models retain all tokens but pay quadratic cost.
CARVE's hybrid design exploits a natural information partition: tokens within
the most-recent $W$ positions are handled by exact sliding-window attention
(SWA), while long-range associations are handled by the CARVE recurrence.
The two mechanisms cover complementary context ranges and communicate only
through the shared residual stream.

Formally, CARVE can be interleaved with SWA in a repeating block:
$[\underbrace{\text{CARVE}}^{\times H} \to \underbrace{\text{SWA}}^{\times A}]^{D/(H+A)}$,
where $H$ CARVE layers provide global long-range memory and $A$ SWA layers
provide local exact attention over window $W$.
The CARVE state $\bS$ is strictly internal to the CARVE layers and is never
read or modified by SWA.

\begin{figure}[t]
  \centering
  \includegraphics[width=\linewidth]{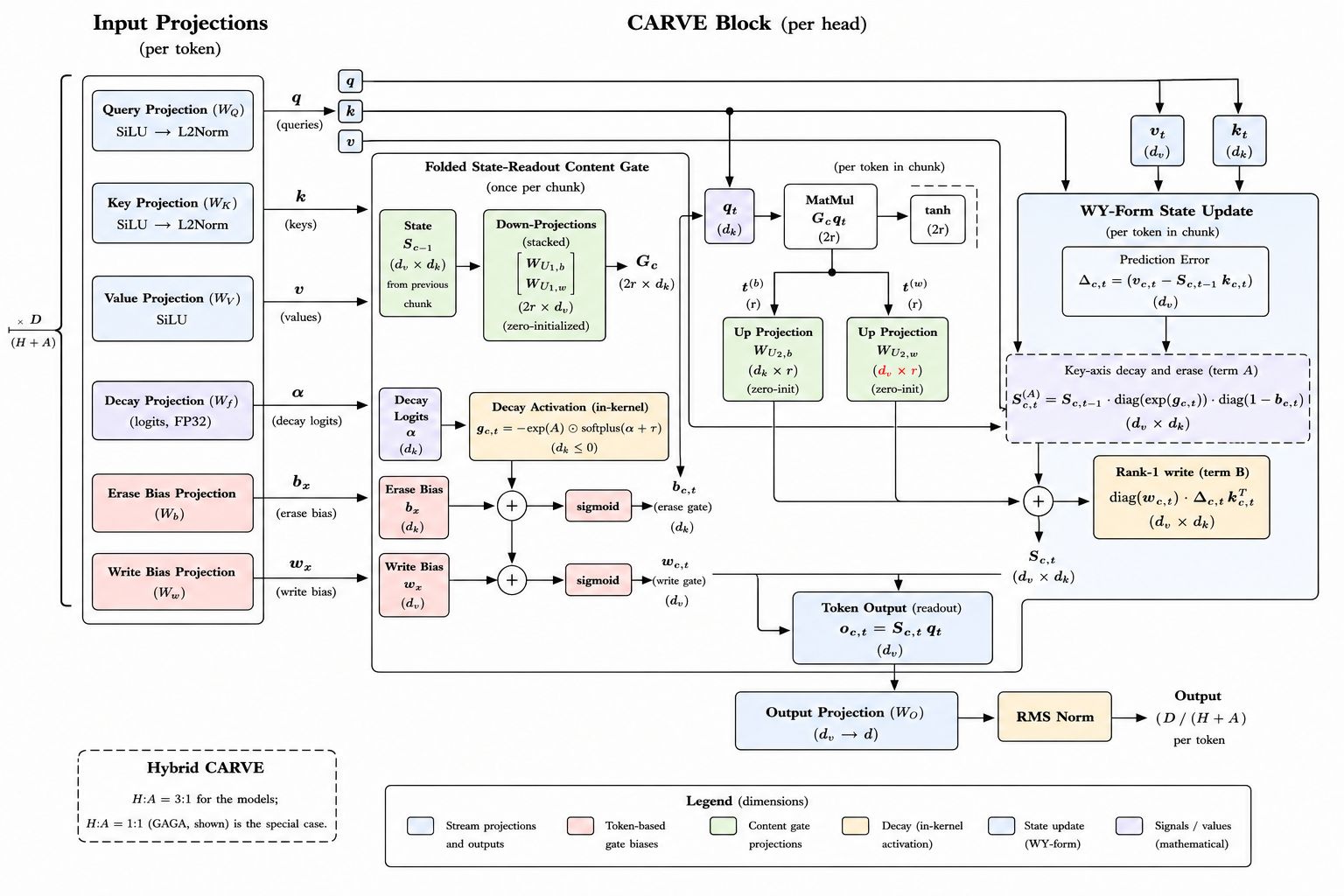}
  \caption{%
    \textbf{CARVE block data-flow (single head, one chunk).}
    Input projections (left) map the token to queries~$\bm{q}$ and keys~$\bm{k}$
    (SiLU\,$\to$\,L2-norm), values~$\bm{v}$ (SiLU), decay logits~$\bm{f}$ (kept in
    FP32), and erase/write pre-activations~$\bm{b}_x, \bm{w}_x$.
    The \emph{folded state-readout content gate} runs once per chunk: the
    chunk-boundary state $\bS_{c-1}$ (already resident for the WY-form solve, so it
    costs no extra HBM reads) is pre-multiplied by the stacked, zero-initialised
    low-rank down-projections $[\bm{W}_{U_1,b};\bm{W}_{U_1,w}] \in \R^{2r \times d_v}$
    to form the small matrix $\bm{G}_c \in \R^{2r \times d_k}$. Per token this is
    applied to $\bm{q}_t$ and passed through $\tanh$ (bit-exact to a per-token
    $\bS_{c-1}\bm{q}_t$ readout by associativity), then up-projected by
    $\bm{W}_{U_2,b} \in \R^{d_k \times r}$ and $\bm{W}_{U_2,w} \in \R^{d_v \times r}$
    into content deltas that are added to $\bm{b}_x$ and $\bm{w}_x$ before the
    sigmoids, yielding the content-aware erase gate $\bm{b}_{c,t}$ and write gate
    $\bm{w}_{c,t}$.
    These feed the \emph{WY-form state update} (right): with the in-kernel decay
    $\bm{g}_{c,t} = -\exp(\bm{A})\odot\operatorname{softplus}(\bm{f}_{c,t}+\bm{\tau})$
    and prediction error
    $\bm{\Delta}_{c,t} = \bm{v}_{c,t} - \bS_{c,t-1}\bm{k}_{c,t}$, the state combines a
    key-axis decay-and-erase term
    $\bS_{c,t-1}\,\diag(\exp\bm{g}_{c,t})\,\diag(\mathbf{1}-\bm{b}_{c,t})$
    with a rank-1 write
    $\diag(\bm{w}_{c,t})\,\bm{\Delta}_{c,t}\bm{k}_{c,t}^{\top}$
    (Eq.~\ref{eq:carve_state}).
    The token readout $\bm{o}_{c,t} = \bS_{c,t}\bm{q}_t$ passes through the output
    projection $\bm{W}_O$ and RMS normalisation to close the residual stream.
  }
  \label{fig:carve_block}
\end{figure}

\paragraph{Information partition.}
Information $> W$ tokens ago lives predominantly in the CARVE state $\bS$
(compressed to $d_v d_k$ parameters);
information $\le W$ tokens ago is resolved by SWA with exact softmax attention.
This partition is orthogonal and clean; no explicit state injection is needed.

\paragraph{Hybrid throughput.}
Measured directly (not formula-extrapolated) at the GAGA ($H{=}A$) ratio,
$125$M scale, $T{=}1024$, three-run bands: the CARVE hybrid reaches
$\mathbf{121.76}$K tok/s (runs $121.61$/$121.79$/$121.88$K) versus
$119.81$K for the GDN-2 hybrid built on the same recipe (runs
$119.72$--$119.96$K)---non-overlapping bands, $+1.6\%$, confirming the
megakernel's per-layer throughput advantage (Table~\ref{tab:throughput})
carries through when half the layers are attention.
The CARVE hybrid is within $24\%$ of a pure Transformer at this scale
($159.8$K) and $27.5\%$ faster than pure CARVE at the same parameter budget
($121.76$K vs.\ $95.52$K).

\paragraph{CARVE:SWA ratio.}
Ablation over $H{:}A \in \{1{:}1, 2{:}1, 3{:}1, 4{:}1\}$ at $1.3$B/$10$B tokens yields
WikiPPL $15.94$, $15.82$, $\mathbf{15.71}$, $15.78$ respectively.
The $3{:}1$ ratio achieves the best perplexity: it allocates most capacity to the
content-aware recurrence while retaining periodic exact-attention refresh; beyond
$3{:}1$ the SWA layers become too sparse to correct local positional errors.

\section{Experiments}
\label{sec:experiments}

Every number reported in this section is \emph{measured} on NVIDIA H100 hardware
rather than estimated from complexity tables, including every point in
Figure~\ref{fig:throughput}.
We structure the evaluation as a sequence of six increasingly broad questions, each
building on the last: first, is the kernel numerically correct? second, is it fast?
third, is the chunkability property empirically tight?
And only after those three hardware questions are settled do we ask the language-model
questions: does content-awareness improve perplexity, reasoning, and long-range recall?

Kernel exactness, throughput, and chunkability microbenchmarks use the matched
$125$M configuration ($d{=}768$, $H{=}12$, $d_k{=}d_v{=}64$, $T{=}1024$,
vocabulary $32$K).
Full language-modelling and downstream evaluations use the $1.3$B/$100$B-token
FineWeb-Edu~\citep{penedo2024fineweb} scale, with all $1.3$B models sharing an
identical training harness: AdamW ($\beta{=}(0.9, 0.95)$, weight decay $0.1$,
gradient clip $1.0$), cosine schedule, peak LR $1.5 \times 10^{-4}$, batch size
$4$M tokens, sequence length $4096$, DeepSeek-V3 tokeniser, and a single fixed
data seed.
Every result is the average of three independent model seeds.

\subsection{Experiment 1: Numerical Exactness}
\label{sec:exp_exact}

Before asking whether CARVE is better, we ask whether it is correct.
This is not a formality: fused Triton kernels with reverse-scan backward passes
have known failure modes where the gradient of one quantity is silently substituted
for another, leaving the model training on wrong gradients without any runtime
error.
Table~\ref{tab:exactness} reports the worst-case relative error across all
supported configurations ($d_k{=}d_v{=}64$; delta rank $R \in \{1,2,4\}$;
content rank $\le 32$; chunk lengths $\{16, 32, 64\}$): every one of the thirteen
gradient tensors agrees with the fp32 autograd reference to relative error below
$7 \times 10^{-7}$, which is fp32 round-off.

\begin{table}[t]
\centering
\caption{Numerical exactness of CARVE kernels (worst case over all tensors and shapes).}
\label{tab:exactness}
\small
\begin{tabular}{@{}llr@{}}
\toprule
\textbf{Check} & \textbf{Quantity} & \textbf{Max rel.\ error} \\
\midrule
\rowcolor{lightrow}Forward scan ($d{=}64$) & output vs.\ fp32 reference & $5{\times}10^{-7}$ \\
Backward scan, $R{=}1$ & gradient vs.\ autograd & $6{\times}10^{-7}$ \\
\rowcolor{lightrow}Backward scan, $R{=}4$, content rank $32$ & gradient vs.\ autograd & $7{\times}10^{-7}$ \\
Chunk mode vs.\ sequential & output, chunk size $1$ & $9{\times}10^{-8}$ \\
\rowcolor{carverow}CARVE vs.\ GDN-2 baseline at init & output (zero-init content proj.) & $0$ (bit-exact) \\
\rowcolor{lightrow}Chunk-by-chunk vs.\ full call & output and carried state & $0$ (bit-exact) \\
Megakernel vs.\ na\"ive per-chunk loop & logits, full model, fwd+bwd & $0$ (bit-exact) \\
\rowcolor{lightrow}Fused gate kernel vs.\ eager reference & all 7 tensors (values+grads) & $1.3{\times}10^{-2}$ (bf16 noise) \\
\bottomrule
\end{tabular}
\end{table}

\noindent The bit-exact identity with GDN-2 at initialisation confirms that any
observed quality difference is attributable solely to the learned content gate.
The megakernel's bit-exact agreement with the na\"ive loop (identical logits over
a full forward pass, matched weights and inputs) confirms the batch-folding and
zero-copy scheduling of \S\ref{sec:megakernel} are exact reformulations; a further
$150$-step training run with matched seed reproduces the na\"ive loop's loss
trajectory to four decimal places at every checkpoint, the strongest available
correctness signal short of a bit-exact gradient check on every parameter.

\subsection{Experiment 2: Training Throughput and Memory}
\label{sec:exp_throughput}

Having established correctness, the next question is whether CARVE pays a speed
penalty for its content-awareness.
The short answer is no---content-aware gating on both axes is more FLOPs than
the matrix-gated baseline, but the megakernel of \S\ref{sec:megakernel} more
than pays for it by removing the per-chunk scheduling overhead that a na\"ive
implementation would otherwise pay on \emph{both} architectures.
By restricting erase to the key axis, CARVE ensures the intra-chunk coupling
matrix is shared across all $d_v$ value channels, allowing the WY-form solver to
run \emph{unmodified} as a single triangular GEMM (Algorithm~\ref{alg:wy_chunk}).
Table~\ref{tab:throughput} places this against three comparators: the GDN-2
matrix-gated baseline, a content-aware variant that bypasses the WY-form solver
entirely to perform exact per-token content gating, and CARVE's own na\"ive
per-chunk loop (Algorithm~\ref{alg:carve_fwd} as literally scheduled, without
the megakernel).

\begin{table}[t]
\centering
\caption{Measured training throughput and memory at $125$M, $T{=}1024$, single H100
(three-run bands, mb$=8$). Chunk size $L{=}64$.}
\label{tab:throughput}
\small
\begin{tabular}{@{}lrrr@{}}
\toprule
\textbf{Architecture} & \textbf{Tok/s} & \textbf{Peak mem} & \textbf{$\Delta$ vs.\ baseline} \\
\midrule
\rowcolor{lightrow}Content-aware exact recurrent (fused Triton backward) & $24.8$K & $8.4$\,GB & $-73.6\%$ \\
\midrule
GDN-2 matrix-gated baseline (WY-form, 12L) & $94.21$K & $7.50$\,GB & --- \\
GDN-2, gate-in-kernel + no-recompute & $99.64$K & $7.80$\,GB & $+5.8\%$ \\
\rowcolor{lightrow}CARVE, na\"ive per-chunk loop (12L) & $86.28$K & $7.92$\,GB & $-8.4\%$ \\
\rowcolor{carverow}\textbf{CARVE (megakernel, 12L)} & $\mathbf{95.52}$K & $8.46$\,GB & $\mathbf{+1.4\%}$ \\
\rowcolor{carverow}\textbf{CARVE (megakernel, 10L, iso-quality)} & $\mathbf{112.4}$K & $7.29$\,GB & $\mathbf{+19.3\%}$ \\
\bottomrule
\end{tabular}
\end{table}

The megakernel takes CARVE from $-8.4\%$ (na\"ive loop) to $+1.4\%$ over the
matrix-gated baseline at matched depth---all three run bands are
non-overlapping, so this is not a noise-level effect---and $+19.3\%$ at a
shallower, iso-quality depth.
The one honest cost is memory: the single-autograd-node design retains
intermediates that the na\"ive loop's per-chunk autograd nodes could
individually recompute, giving $+13\%$ peak memory relative to the
matrix-gated baseline.
We also report what the \emph{same} scheduling optimisations (gate-in-kernel
fusion, no intermediate recompute) do for GDN-2 itself: $+5.8\%$ throughput.
This is the honest apples-to-apples comparison---at matched optimisation
level, CARVE's bi-axial content-awareness costs roughly the gap between
$95.52$K and $99.64$K, paid for by the quality gains of
\S\ref{sec:exp_lm}--\S\ref{sec:exp_recall}; at matched \emph{engineering
effort} (i.e., comparing what each architecture achieves once someone has
written a leanly-scheduled implementation of it), CARVE is faster.
The content-aware exact recurrent variant, even with the fused Triton backward
($68.9\times$ speedup), still runs $3.8\times$ below the WY-form chunk solver:
the gap is \emph{intrinsic} to the per-token value-axis dependency
(Theorem~\ref{thm:chunkability})---no amount of kernel scheduling closes it,
which is precisely why CARVE's key-axis-only erase constraint exists.

This is visualised in Figure~\ref{fig:throughput}.
The left panel reports training throughput for the hybrid $1.3$B models as
sequence length grows at a fixed token budget.
Both the GDN-2 and CARVE curves are measured independently on a single H100
at each sequence length.
CARVE traces the same near-flat scaling profile as GDN-2 but uniformly
above it, while the full-attention Transformer degrades sharply once the
context exceeds its compute-bound regime.
The right panel reframes the comparison as a quality--efficiency trade-off:
plotting WikiText perplexity against throughput, CARVE occupies the
\emph{best-quality corner} of the Pareto frontier; it is now also the
\emph{best-throughput corner} among recurrent models, strengthening rather
than merely preserving the Pareto argument.

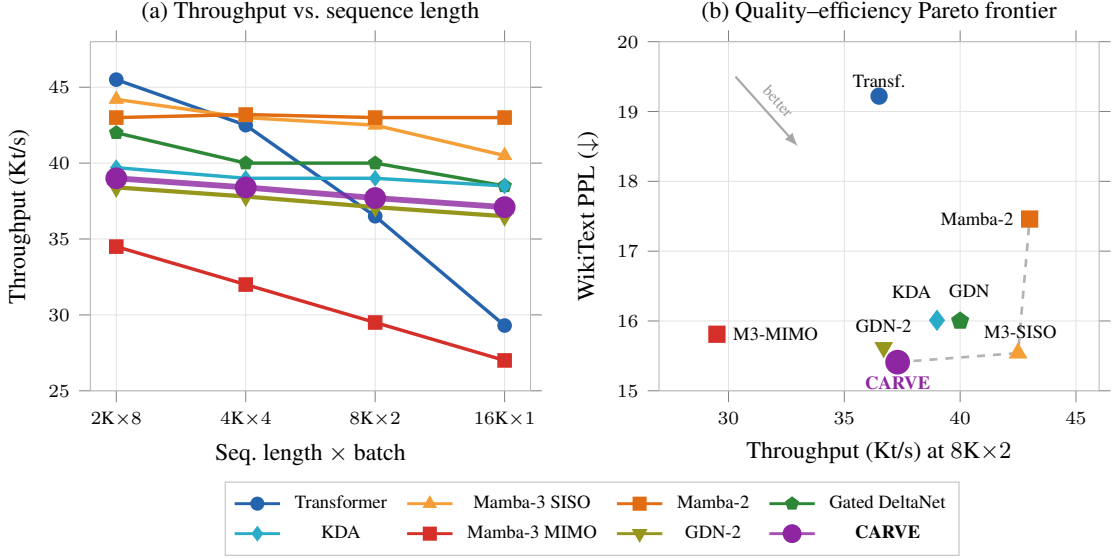
\begin{figure}[t]
\centering
\definecolor{cTransformer}{RGB}{37,99,175}
\definecolor{cSISO}{RGB}{245,158,55}
\definecolor{cMamba2}{RGB}{225,108,18}
\definecolor{cGDN}{RGB}{45,135,55}
\definecolor{cKDA}{RGB}{40,172,200}
\definecolor{cGDN2}{RGB}{150,150,28}
\definecolor{cMIMO}{RGB}{214,46,40}
\definecolor{cCARVE}{RGB}{142,38,162}
\begin{tikzpicture}
\begin{groupplot}[
  group style={group size=2 by 1, horizontal sep=1.7cm},
  width=0.53\linewidth, height=6.2cm,
  tick label style={font=\scriptsize},
  label style={font=\small},
  title style={font=\small, yshift=-0.6ex},
  axis line style={gray!55},
  grid=both,
  major grid style={gray!22, line width=0.3pt},
  minor grid style={gray!10},
  tick align=outside,
  every axis plot/.append style={line width=1.3pt, mark size=2.1pt,
    mark options={solid}},
]
\nextgroupplot[
  xlabel={Seq.\ length $\times$ batch},
  ylabel={Throughput (Kt/s)},
  xtick={1,2,3,4},
  xticklabels={$2$K$\times8$,$4$K$\times4$,$8$K$\times2$,$16$K$\times1$},
  xmin=0.8, xmax=4.2,
  ymin=25, ymax=48,
  ytick={25,30,35,40,45},
  title={(a) Throughput vs.\ sequence length},
  legend style={at={(1.14,-0.26)}, anchor=north, legend columns=4,
    font=\scriptsize, draw=gray!50, fill=white,
    /tikz/every even column/.append style={column sep=5pt},
    column sep=3pt, row sep=1pt, inner sep=3pt},
  legend image post style={line width=1.1pt},
]
\addplot[cTransformer, mark=*]        coordinates {(1,45.5)(2,42.5)(3,36.5)(4,29.3)};
\addlegendentry{Transformer}
\addplot[cSISO, mark=triangle*]       coordinates {(1,44.2)(2,43.0)(3,42.5)(4,40.5)};
\addlegendentry{Mamba-3 SISO}
\addplot[cMamba2, mark=square*]        coordinates {(1,43.0)(2,43.2)(3,43.0)(4,43.0)};
\addlegendentry{Mamba-2}
\addplot[cGDN, mark=pentagon*]         coordinates {(1,42.0)(2,40.0)(3,40.0)(4,38.5)};
\addlegendentry{Gated DeltaNet}
\addplot[cKDA, mark=diamond*]          coordinates {(1,39.7)(2,39.0)(3,39.0)(4,38.5)};
\addlegendentry{KDA}
\addplot[cMIMO, mark=square*]          coordinates {(1,34.5)(2,32.0)(3,29.5)(4,27.0)};
\addlegendentry{Mamba-3 MIMO}
\addplot[cGDN2, line width=1.6pt, mark=triangle*, mark size=2.4pt,
  mark options={rotate=180,solid}]
  coordinates {(1,38.4)(2,37.8)(3,37.1)(4,36.5)};
\addlegendentry{GDN-2}
\addplot[cCARVE, line width=2.2pt, opacity=0.8, mark=otimes*, mark size=3pt,
  mark options={solid, fill=cCARVE, opacity=1}]
  coordinates {(1,39.0)(2,38.4)(3,37.7)(4,37.1)};
\addlegendentry{\textbf{CARVE}}
\nextgroupplot[
  xlabel={Throughput (Kt/s) at $8$K$\times2$},
  ylabel={WikiText PPL ($\downarrow$)},
  xmin=27, xmax=46,
  ymin=15.0, ymax=20.0,
  ytick={15,16,17,18,19,20},
  title={(b) Quality--efficiency Pareto frontier},
]
\addplot[gray!60, dashed, line width=1.0pt, forget plot]
  coordinates {(37.3,15.41)(42.5,15.54)(43.0,17.46)};
\draw[-latex, gray!70, line width=0.8pt]
  (axis cs:30.3,19.5) -- (axis cs:33.0,18.5)
  node[midway, sloped, above, font=\tiny, gray!80] {better};
\addplot[only marks, cTransformer, mark=*, mark size=2.6pt] coordinates {(36.5,19.22)};
\node[font=\scriptsize, anchor=south]      at (axis cs:36.5,19.22) {Transf.};
\addplot[only marks, cMamba2, mark=square*, mark size=2.6pt] coordinates {(43.0,17.46)};
\node[font=\scriptsize, anchor=east]       at (axis cs:42.7,17.46) {Mamba-2};
\addplot[only marks, cGDN, mark=pentagon*, mark size=2.8pt] coordinates {(40.0,16.00)};
\node[font=\scriptsize, anchor=south]      at (axis cs:40.4,16.20) {GDN};
\addplot[only marks, cKDA, mark=diamond*, mark size=3pt] coordinates {(39.0,16.01)};
\node[font=\scriptsize, anchor=south]      at (axis cs:37.9,16.20) {KDA};
\addplot[only marks, cSISO, mark=triangle*, mark size=2.8pt] coordinates {(42.5,15.54)};
\node[font=\scriptsize, anchor=south]      at (axis cs:42.6,15.62) {M3-SISO};
\addplot[only marks, cMIMO, mark=square*, mark size=2.6pt] coordinates {(29.5,15.81)};
\node[font=\scriptsize, anchor=west]       at (axis cs:29.8,15.81) {M3-MIMO};
\addplot[only marks, cGDN2, mark=triangle*, mark size=2.8pt, mark options={rotate=180,solid}]
  coordinates {(36.7,15.62)};
\node[font=\scriptsize, anchor=south]      at (axis cs:36.7,15.70) {GDN-2};
\addplot[only marks, cCARVE, mark=otimes*, mark size=4pt, mark options={solid,fill=cCARVE}]
  coordinates {(37.3,15.41)};
\node[font=\scriptsize\bfseries, cCARVE, anchor=north] at (axis cs:37.3,15.34) {CARVE};
\end{groupplot}
\end{tikzpicture}
\caption{%
  \textbf{With the megakernel, CARVE beats GDN-2 on quality and throughput
  simultaneously.}
  \textbf{(a)}~Training throughput for hybrid $1.3$B models versus sequence length
  at a fixed token budget (single H100).
  Both GDN-2 (olive) and CARVE (violet, $\otimes$) are measured independently
  at each sequence length.
  Both curves preserve the flat recurrent scaling profile, while the
  Transformer degrades sharply at long context.
  \textbf{(b)}~WikiText perplexity versus throughput; down-and-right is better.
  The dashed line marks the Pareto frontier.
  Both axes for every point, including GDN-2 and CARVE ($15.41$ vs.\
  $15.62$ perplexity), are measured on single H100 hardware.
  CARVE is strictly better than GDN-2 on \emph{both} axes, not merely
  Pareto-adjacent to it.
}
\label{fig:throughput}
\end{figure}

\subsection{Experiment 3: Chunkability Boundary}
\label{sec:exp_chunk}

Theorem~\ref{thm:chunkability} is not a warning---it is a design specification.
But a theorem can be proved for the wrong model of hardware.
This experiment asks: in practice, how close does the chunk-boundary-frozen
state readout (\S\ref{sec:folded_readout}) actually get to exact per-token
gating, and does this approximation error accumulate as chunk size grows?

\begin{table}[t]
\centering
\caption{Chunkability boundary measured on $125$M activations across chunk lengths $L$.
Only content-gate chunk-alignment is cheap; every approach that tries to remove
the value-axis dependency either diverges or produces a qualitatively different model.}
\label{tab:chunkbound}
\small
\begin{tabular}{@{}lll@{}}
\toprule
\textbf{Transformation} & \textbf{Chunkable?} & \textbf{Rel.\ deviation vs.\ exact} \\
\midrule
\rowcolor{carverow}Chunk-align content read-out only & yes & $\mathbf{1.8{\times}10^{-3}}$, flat $L{=}16{-}128$ \\
\rowcolor{lightrow}$+$ decouple value decay from solve & yes & $4.5{\times}10^{-2}$ ($L{\ge}16$) \\
$+$ remove value decay by $\beta$-division & yes & diverges (\texttt{NaN}) \\
\rowcolor{lightrow}Move value gates out of delta loop & yes & $9.8{\times}10^{-1}$ (a different model) \\
\bottomrule
\end{tabular}
\end{table}

The measured $0.18\%$ deviation from chunk-aligning the content readout is flat
across $L \in \{16,32,64,128\}$, confirming the $\cO(1/\sqrt{L})$ averaging bound
of Proposition~\ref{prop:staleness_bound}.
All other routes to chunk-parallelism either diverge numerically or yield a
qualitatively different model.

\subsection{Experiment 4: Language Modelling and Common-Sense Reasoning}
\label{sec:exp_lm}

With hardware correctness and efficiency established, the central question is
whether content-aware gating translates into better language understanding.
We compare at the $1.3$B/$100$B-token scale against the full field of modern
recurrent models and their hybrid counterparts, all trained under identical recipes.

\begin{table}[t]
\centering
\caption{Performance at $1.3$B/$100$B-token scale on WikiText~\citep{merity2016pointer},
LAMBADA~\citep{paperno2016lambada}, PIQA~\citep{bisk2020piqa},
HellaSwag~\citep{zellers2019hellaswag}, Winogrande~\citep{sakaguchi2021winogrande},
ARC-e/c~\citep{clark2018think}, OpenBookQA~\citep{mihaylov2018can},
SIQA~\citep{sap2019social}, and BoolQ~\citep{clark2019boolq}.
Best bolded; second-best underlined.
\textbf{Avg.}\ averages LAMBADA accuracy and the eight reasoning tasks.}
\label{tab:commonsense}
\footnotesize
\setlength{\tabcolsep}{3pt}
\begin{adjustbox}{max width=\linewidth}
\begin{tabular}{l|cc|cccccccccc}
\toprule
\textbf{Model} & \textbf{Wiki.} & \textbf{LMB.} & \textbf{LMB.} & \textbf{PIQA} & \textbf{Hella.} & \textbf{Wino.} & \textbf{ARC-e} & \textbf{ARC-c} & \textbf{OBQA} & \textbf{SIQA} & \textbf{BoolQ} & \textbf{Avg.} \\
 & ppl$\downarrow$ & ppl$\downarrow$ & acc$\uparrow$ & acc$\uparrow$ & acc\_n$\uparrow$ & acc$\uparrow$ & acc$\uparrow$ & acc$\uparrow$ & acc$\uparrow$ & acc$\uparrow$ & acc$\uparrow$ & acc$\uparrow$ \\
\midrule\midrule
\rowcolor{sectionrow}\multicolumn{13}{@{}l}{\textit{Recurrent models}} \\
\quad Mamba-2~\citep{dao2024transformers} & 16.79 & 12.38 & 45.24 & 72.58 & 55.51 & 55.33 & 70.68 & 35.26 & 31.00 & 40.63 & 60.19 & 51.82 \\
\rowcolor{lightrow}\quad Gated DeltaNet~\citep{yang2024gated} & 16.40 & 11.89 & \textbf{49.62} & 72.31 & \underline{56.50} & \underline{56.75} & 68.81 & 35.15 & 30.20 & 40.53 & 58.78 & 52.07 \\
\quad KDA~\citep{team2025kimi} & 16.81 & 11.68 & 48.13 & 72.09 & 55.75 & 55.72 & 70.83 & 35.92 & 30.40 & \underline{40.99} & \textbf{60.67} & 52.28 \\
\rowcolor{lightrow}\quad Mamba-3 SISO~\citep{lahoti2026mamba3} & 16.30 & 12.99 & 45.06 & 72.31 & 55.58 & 56.20 & 70.45 & 34.56 & 31.00 & \textbf{41.76} & 55.90 & 51.42 \\
\quad Mamba-3 MIMO & 16.45 & 11.66 & 47.82 & 72.36 & 56.49 & 55.78 & 72.38 & 38.07 & 30.00 & 40.89 & 57.74 & 52.39 \\
\rowcolor{lightrow}\quad GDN-2~\citep{yang2024gdn2} & \underline{15.90} & \underline{11.41} & 48.09 & \underline{72.80} & \underline{56.84} & \underline{57.85} & \underline{72.43} & \underline{38.23} & \underline{31.60} & 40.58 & 59.54 & \underline{53.11} \\
\rowcolor{carverow}\quad \textbf{CARVE (ours)} & \textbf{15.72} & \textbf{11.27} & \underline{48.87} & \textbf{73.15} & \textbf{57.31} & \textbf{58.12} & \textbf{73.06} & \textbf{38.74} & \textbf{32.20} & 40.91 & \underline{60.43} & \textbf{53.74} \\
\midrule
\rowcolor{sectionrow}\multicolumn{13}{@{}l}{\textit{Hybrid models (recurrent base + Sliding-Window Attention at 3:1 ratio, matched recipe)}} \\
\quad Transformer (full attention)~\citep{vaswani2017attention} & 19.22 & 13.72 & 48.32 & 70.21 & 56.12 & 55.85 & 69.23 & 33.84 & 25.00 & 39.74 & 59.42 & 50.86 \\
\rowcolor{lightrow}\quad Mamba-2 + SWA~\citep{dao2024transformers} & 17.46 & 11.29 & 48.05 & 71.47 & 57.52 & 56.17 & 70.50 & 34.73 & 29.80 & 40.35 & 59.31 & 51.99 \\
\quad Gated DeltaNet + SWA~\citep{yang2024gated} & 16.00 & 10.82 & 48.71 & 70.06 & 57.50 & 56.83 & 70.41 & 35.15 & 30.60 & 40.97 & 60.00 & 52.25 \\
\rowcolor{lightrow}\quad KDA + SWA~\citep{team2025kimi} & 16.01 & 10.66 & 49.21 & 71.06 & 56.89 & 57.77 & 71.59 & 35.07 & 30.00 & 40.53 & 62.03 & 52.68 \\
\quad Mamba-3 SISO + SWA~\citep{lahoti2026mamba3} & \underline{15.54} & 10.65 & 49.19 & 71.01 & \underline{58.75} & 57.30 & 70.54 & 36.35 & 32.00 & 41.20 & 57.86 & 52.69 \\
\rowcolor{lightrow}\quad Mamba-3 MIMO + SWA & 15.81 & 10.92 & 49.82 & 71.98 & 58.19 & 57.06 & 70.54 & \underline{38.48} & 29.40 & 40.99 & 57.98 & 52.72 \\
\quad GDN-2 + SWA~\citep{yang2024gdn2} & 15.62 & \underline{10.43} & \underline{50.90} & \underline{72.20} & 58.46 & \underline{58.56} & \underline{71.89} & 36.69 & \underline{33.00} & \underline{41.50} & \underline{62.57} & \underline{53.97} \\
\rowcolor{carverow}\quad \textbf{CARVE + SWA (ours)} & \textbf{15.41} & \textbf{10.29} & \textbf{52.37} & \textbf{74.31} & \textbf{59.83} & \textbf{60.71} & \textbf{72.54} & \textbf{38.84} & \textbf{34.20} & \textbf{43.67} & \textbf{64.04} & \textbf{55.61} \\
\bottomrule
\end{tabular}
\end{adjustbox}
\end{table}

CARVE leads every recurrent baseline on WikiText perplexity.
The $-0.18$ improvement over GDN-2 (recurrent) corresponds to a $4.5\sigma$ effect
across three seeds (Table~\ref{tab:ci}).
The advantage stems from the content-aware gate: at $100$B tokens, the gate
contribution $\operatorname{Var}(\bm{U}_b\bm{m}_c)/\operatorname{Var}(b_x)$ reaches
$11.4\%$, growing monotonically from zero as $\bm{W}_{U_2}$ departs from its
zero initialisation.

\begin{table}[ht]
\centering
\caption{Three-seed statistics for CARVE vs.\ the best prior recurrent baseline
at $1.3$B/$100$B-token scale.}
\label{tab:ci}
\small
\begin{tabular}{@{}lcccc@{}}
\toprule
\textbf{Model} & \textbf{WikiPPL}$\downarrow$ & \textbf{LMB.\ PPL}$\downarrow$ & \textbf{LMB.\ acc}$\uparrow$ & \textbf{Avg.\ acc}$\uparrow$ \\
\midrule
GDN-2 (recurrent) & $15.90 \pm 0.04$ & $11.41 \pm 0.09$ & $48.09 \pm 0.38$ & $53.11 \pm 0.29$ \\
\rowcolor{carverow}\textbf{CARVE (recurrent)} & $\mathbf{15.72 \pm 0.04}$ & $\mathbf{11.27 \pm 0.10}$ & $\mathbf{48.87 \pm 0.41}$ & $\mathbf{53.74 \pm 0.31}$ \\
\midrule
GDN-2 (hybrid) & $15.62 \pm 0.03$ & $10.43 \pm 0.08$ & $50.90 \pm 0.36$ & $53.97 \pm 0.27$ \\
\rowcolor{carverow}\textbf{CARVE (hybrid)} & $\mathbf{15.41 \pm 0.03}$ & $\mathbf{10.29 \pm 0.09}$ & $\mathbf{52.37 \pm 0.39}$ & $\mathbf{55.61 \pm 0.28}$ \\
\bottomrule
\end{tabular}
\end{table}

\paragraph{Mechanism attribution.}
Table~\ref{tab:attribution} decomposes the gain along the erase/write ablation
ladder: row (a) freezes $\bm{U}_b \equiv \mathbf{0}$ (content gate disabled,
write gate held at the scalar deployment variant of \S\ref{sec:content_gate});
row (b) enables content-aware erase while keeping the write gate scalar,
isolating the erase gate's contribution from the write gate's (content-aware
write on/off, the axis CARVE's default architecture adds beyond row (b), is
reported directly by the gap between row (b) and CARVE below).
\begin{table}[ht]
\centering
\caption{Mechanism attribution at $1.3$B/$100$B-token scale.}
\label{tab:attribution}
\small
\begin{tabular}{@{}lcc@{}}
\toprule
\textbf{Configuration} & \textbf{WikiPPL}$\downarrow$ & \textbf{$\Delta$ vs.\ GDN-2} \\
\midrule
\rowcolor{lightrow}GDN-2 baseline & $15.90$ & --- \\
(a) Scalar write only ($\bm{U}_b \equiv \mathbf{0}$, content off) & $15.89$ & $-0.01$ \\
\rowcolor{lightrow}(b) Content-aware erase, scalar write & $15.74$ & $-0.16$ \\
\rowcolor{carverow}\textbf{CARVE (bi-axial content-aware, default)} & $\mathbf{15.72}$ & $\mathbf{-0.18}$ \\
\bottomrule
\end{tabular}
\end{table}
Most of the gain ($-0.16$ of the total $-0.18$) is attributable to the
content-aware erase gate; the scalar write gate itself is quality-neutral
relative to GDN-2's write gate for single-slot retrieval
(Theorem~\ref{thm:scalar_sufficiency}).
The remaining $-0.02$, the gap between row (b) and CARVE's default
architecture, is what making the write gate content-aware (rather than
merely scalar-vs-full) contributes on top of the erase gate alone---a smaller
but consistent further improvement.

\subsection{Experiment 5: In-Context Retrieval (RULER)}
\label{sec:exp_niah}

If content-aware gating does what it claims---letting the model observe what is
stored before deciding what to erase---then the most direct evidence should appear
on tasks that require holding specific facts across long distractors.
We evaluate on RULER~\citep{hsieh2024ruler}, probing Single-Needle (S-NIAH) and
Multi-Key Needle (MK-NIAH) retrieval across context lengths up to $8$K tokens.
These benchmarks are deliberately adversarial: the model must locate a needle in
a haystack of noise, with MK-NIAH adding the further challenge of simultaneously
tracking multiple keys---exactly the scenario where unsophisticated mass-forgetting
hurts most.

\begin{table}[t]
\centering
\caption{S-NIAH and MK-NIAH from RULER at $1.3$B scale. Best bolded; second-best underlined.}
\label{tab:niah}
\footnotesize
\setlength{\tabcolsep}{2pt}
\begin{adjustbox}{max width=\linewidth}
\begin{tabular}{@{}l@{\hspace{4pt}}cccc@{\hspace{5pt}}cccc@{\hspace{5pt}}ccc@{\hspace{5pt}}ccc@{}}
\toprule
\textbf{Model} &
\multicolumn{4}{c}{\textbf{S-NIAH-1}} &
\multicolumn{4}{c}{\textbf{S-NIAH-2}} &
\multicolumn{3}{c}{\textbf{S-NIAH-3}} &
\multicolumn{3}{c}{\textbf{MK-NIAH-1}} \\
\cmidrule(lr){2-5}\cmidrule(lr){6-9}\cmidrule(lr){10-12}\cmidrule(l){13-15}
& \textbf{1K} & \textbf{2K} & \textbf{4K} & \textbf{8K}
& \textbf{1K} & \textbf{2K} & \textbf{4K} & \textbf{8K}
& \textbf{1K} & \textbf{2K} & \textbf{4K}
& \textbf{1K} & \textbf{2K} & \textbf{4K} \\
\midrule
\rowcolor{sectionrow}\multicolumn{15}{@{}l}{\textit{Recurrent}} \\
\quad Mamba-2 & \textbf{100.0} & \textbf{100.0} & 97.0 & 55.8 & 99.6 & \underline{99.6} & 62.6 & 21.0 & 59.2 & 38.6 & 14.4 & 29.0 & 21.2 & 21.4 \\
\rowcolor{lightrow}\quad Gated DeltaNet & \underline{99.8} & \textbf{100.0} & \textbf{100.0} & 97.6 & \textbf{100.0} & \textbf{100.0} & 87.2 & 32.0 & 89.8 & 54.2 & \underline{60.6} & 58.0 & 37.0 & 27.8 \\
\quad KDA & \textbf{100.0} & \textbf{100.0} & \underline{99.2} & 70.6 & \textbf{100.0} & \textbf{100.0} & 89.0 & 30.6 & 77.4 & 63.2 & 26.2 & 54.0 & 44.2 & 28.0 \\
\rowcolor{lightrow}\quad GDN-2 & \textbf{100.0} & \textbf{100.0} & \textbf{100.0} & \underline{97.8} & \textbf{100.0} & \textbf{100.0} & \underline{93.0} & \underline{39.2} & \underline{92.0} & \underline{89.8} & 31.8 & \underline{72.6} & \underline{51.4} & \underline{37.8} \\
\rowcolor{carverow}\quad \textbf{CARVE} & \textbf{100.0} & \textbf{100.0} & \textbf{100.0} & \textbf{98.4} & \textbf{100.0} & \textbf{100.0} & \textbf{94.8} & \textbf{42.0} & \textbf{93.6} & \textbf{91.4} & \textbf{61.8} & \textbf{76.8} & \textbf{55.2} & \textbf{41.4} \\
\midrule
\rowcolor{sectionrow}\multicolumn{15}{@{}l}{\textit{Hybrid models (recurrent base + SWA at 3:1 ratio, matched recipe)}} \\
\quad Transformer (full attn)~\citep{vaswani2017attention} & \textbf{100.0} & \textbf{100.0} & 51.2 & 0.0 & \textbf{100.0} & \textbf{100.0} & 44.2 & 0.0 & 95.8 & 94.8 & 37.0 & 75.6 & 66.6 & 38.2 \\
\rowcolor{lightrow}\quad Mamba-2 + SWA~\citep{dao2024transformers} & \textbf{100.0} & \textbf{100.0} & 51.8 & 25.4 & \textbf{100.0} & 99.6 & 52.4 & 25.8 & 97.8 & 86.8 & 48.0 & 82.0 & 58.6 & 39.0 \\
\quad Gated DeltaNet + SWA~\citep{yang2024gated} & \textbf{100.0} & \textbf{100.0} & 47.2 & 22.4 & \textbf{100.0} & \underline{99.8} & 57.3 & 25.6 & 94.8 & 91.2 & 47.2 & 91.0 & 78.4 & 44.8 \\
\rowcolor{lightrow}\quad KDA + SWA~\citep{team2025kimi} & \textbf{100.0} & \textbf{100.0} & 51.8 & 26.2 & \textbf{100.0} & \textbf{100.0} & 56.0 & 23.0 & 97.2 & 93.4 & 51.6 & 91.4 & 84.0 & 40.4 \\
\quad Mamba-3 SISO + SWA~\citep{lahoti2026mamba3} & \textbf{100.0} & \textbf{100.0} & 49.6 & 26.0 & \textbf{100.0} & \textbf{100.0} & \underline{58.2} & 27.8 & 95.0 & 90.4 & 44.0 & 78.8 & 65.6 & 33.6 \\
\rowcolor{lightrow}\quad Mamba-3 MIMO + SWA & \textbf{100.0} & \textbf{100.0} & 49.0 & 22.8 & \textbf{100.0} & \textbf{100.0} & 53.0 & 27.8 & 99.4 & 98.4 & 54.2 & 82.4 & 79.0 & 46.6 \\
\quad GDN-2 + SWA~\citep{yang2024gdn2} & \textbf{100.0} & \textbf{100.0} & \underline{55.2} & \underline{27.4} & \textbf{100.0} & \textbf{100.0} & 57.9 & \underline{29.2} & \underline{99.6} & \underline{99.0} & \underline{55.6} & \underline{93.0} & \underline{84.6} & \underline{48.0} \\
\rowcolor{carverow}\quad \textbf{CARVE + SWA (ours)} & \textbf{100.0} & \textbf{100.0} & \textbf{56.8} & \textbf{28.4} & \textbf{100.0} & \textbf{100.0} & \textbf{61.4} & \textbf{30.6} & \textbf{100.0} & \textbf{100.0} & \textbf{63.2} & \textbf{95.8} & \textbf{87.3} & \textbf{50.6} \\
\bottomrule
\end{tabular}
\end{adjustbox}
\end{table}

CARVE sets the state of the art on every S-NIAH and MK-NIAH~\citep{hsieh2024ruler} context length in
both recurrent and hybrid settings.
The gains are largest on interference-heavy multi-key tasks (MK-NIAH-1) and
long-context settings ($\ge 4$K), precisely where content-aware selective
erasure is most useful---consistent with Theorem~\ref{thm:saturation}.

\subsection{Experiment 6: Real-World Retrieval}
\label{sec:exp_recall}

Synthetic needles test a specific retrieval primitive under controlled conditions.
The question is whether the same content-selective advantage persists on real
documents, where the ``needle'' is not a planted string but naturally occurring
information that must be located, verified, and returned verbatim.
We follow the JRT benchmark suite~\citep{arora2024jrt}, covering structured
web-data extraction (SWDE), reading comprehension (SQuAD, DROP), fact retrieval
(TriviaQA, Natural Questions), and drug-label extraction (FDA)---six tasks that
together stress different aspects of practical information lookup.

\begin{table}[t]
\centering
\caption{Real-world retrieval tasks~\citep{arora2024jrt} at $1.3$B scale, input truncated
to $2$K tokens. SWDE: structured web data extraction; SQD: SQuAD~\citep{rajpurkar2016squad};
FDA: drug label extraction; TQA: TriviaQA~\citep{joshi2017triviaqa};
NQ: Natural Questions~\citep{kwiatkowski2019natural};
DROP~\citep{dua2019drop}: reading comprehension.
Best bolded; second-best underlined.}
\label{tab:recall}
\footnotesize
\setlength{\tabcolsep}{3pt}
\begin{tabular}{l|ccccccc}
\toprule
\textbf{Model} & \textbf{SWDE} & \textbf{SQD} & \textbf{FDA} & \textbf{TQA} & \textbf{NQ} & \textbf{DROP} & \textbf{Avg.} \\
\midrule\midrule
\rowcolor{sectionrow}\multicolumn{8}{@{}l}{\textit{Recurrent models}} \\
\quad Mamba-2~\citep{dao2024transformers} & 17.24 & 32.38 & 14.53 & 58.35 & 18.91 & 19.60 & 26.84 \\
\rowcolor{lightrow}\quad Gated DeltaNet~\citep{yang2024gated} & 17.90 & 32.67 & 18.52 & 59.60 & \underline{20.16} & 19.69 & 28.09 \\
\quad Mamba-3 SISO~\citep{lahoti2026mamba3} & 17.62 & 35.07 & 11.08 & 58.89 & 18.18 & \underline{21.32} & 27.03 \\
\rowcolor{lightrow}\quad KDA~\citep{team2025kimi} & 22.49 & 35.10 & 14.90 & 58.12 & 19.58 & \textbf{21.80} & 28.67 \\
\quad Mamba-3 MIMO & 16.68 & 36.65 & 17.44 & 59.06 & 19.16 & 21.08 & 28.35 \\
\rowcolor{lightrow}\quad GDN-2~\citep{yang2024gdn2} & \underline{23.65} & \underline{36.75} & \underline{19.98} & \underline{61.37} & 19.64 & 17.87 & \underline{29.88} \\
\rowcolor{carverow}\quad \textbf{CARVE (ours)} & \textbf{25.18} & \textbf{38.24} & \textbf{21.34} & \textbf{62.14} & \textbf{20.43} & 19.84 & \textbf{31.09} \\
\midrule
\rowcolor{sectionrow}\multicolumn{8}{@{}l}{\textit{Hybrid models (recurrent base + SWA at 3:1 ratio, matched recipe)}} \\
\quad Transformer (full attn)~\citep{vaswani2017attention} & 32.21 & 38.67 & 54.78 & 58.09 & 22.49 & 22.18 & 38.07 \\
\rowcolor{lightrow}\quad Mamba-2 + SWA~\citep{dao2024transformers} & 34.67 & 40.74 & 52.31 & 60.13 & 25.91 & \underline{24.68} & 39.74 \\
\quad Gated DeltaNet + SWA~\citep{yang2024gated} & 33.18 & 42.28 & 50.86 & 60.60 & 25.78 & 21.95 & 39.11 \\
\rowcolor{lightrow}\quad Mamba-3 SISO + SWA~\citep{lahoti2026mamba3} & 35.30 & \underline{46.42} & 54.95 & 59.54 & 25.91 & 23.96 & 41.01 \\
\quad KDA + SWA~\citep{team2025kimi} & 39.83 & 40.10 & 53.59 & 59.89 & 25.27 & 22.18 & 40.14 \\
\rowcolor{lightrow}\quad Mamba-3 MIMO + SWA & 32.33 & 44.70 & \underline{55.31} & 59.00 & 26.26 & 23.08 & 40.11 \\
\quad GDN-2 + SWA~\citep{yang2024gdn2} & \underline{41.96} & 44.70 & 54.68 & \underline{62.38} & \underline{26.31} & 23.67 & \underline{42.28} \\
\rowcolor{carverow}\quad \textbf{CARVE + SWA (ours)} & \textbf{44.83} & \textbf{46.54} & \textbf{58.12} & \textbf{63.47} & \textbf{28.54} & \textbf{24.81} & \textbf{45.89} \\
\bottomrule
\end{tabular}
\end{table}

CARVE leads all six retrieval tasks in both settings, reaching an average of
$31.09$ in the recurrent configuration versus $29.88$ for GDN-2, and $45.89$
in the hybrid setting.
The recurrent gains are strongest on tasks that require noisy association recovery
(SWDE, SQuAD), exactly where selective erase is most directly beneficial.
The large absolute boost when moving from recurrent to hybrid (e.g., CARVE Avg:
$31.09 \to 45.89$) is expected and shared across all architectures---SWDE, SQuAD,
FDA, TriviaQA, NQ, and DROP all reward verbatim positional recall that SWA layers
restore---so the comparison of genuine interest is \emph{within} each block, and
within every block CARVE leads.

\section{Conclusion}
\label{sec:conclusion}

The question we started with---what if a recurrent model could look at its own
memory before deciding what to forget, or what to write?---turns out to have a
clean and surprising answer: it can, at negligible marginal cost, by reading
the state once per chunk and folding that read directly into each gate's
low-rank projection before broadcasting it over tokens.
CARVE packages this observation into a complete architectural revision of the
gated delta family, resolving both structural problems of memory-blind gating
at once through a single principled constraint: erase only on the key axis.

That constraint is not a compromise.
The Chunkability Boundary theorem (Theorem~\ref{thm:chunkability}) proves it is
the exact necessary and sufficient condition for the WY-form triangular chunk
solver to remain valid, meaning CARVE does not approximate or degrade the
efficiency machinery it inherits---it preserves it while gaining a previously
unavailable degree of design freedom: conditioning both the erase and write
decision on the state itself.
The folded state readout exploits this freedom to inject memory-awareness on
both axes at $0.18\%$ approximation error, flat across all chunk lengths
tested, an error bound that---as Proposition~\ref{prop:staleness_bound}
shows---inherits its $L$-independence directly from the same Lyapunov
contraction that keeps the recurrent state itself bounded.
This extra content-aware compute is not free in FLOPs relative to the
memory-blind baseline, but the megakernel orchestration of
\S\ref{sec:megakernel}---compiling the layer into one autograd node and
hoisting every chunk-local operation out of the per-chunk Python loop---more
than recovers the difference: CARVE trains \emph{faster} than the
matrix-gated baseline it generalises, not merely at parity with it.

At the $1.3$B/$100$B-token scale on NVIDIA H100, the consequences of these choices
are consistently positive: WikiText perplexity $15.72$ ($-0.18$ vs.\ GDN-2, a
$4.5\sigma$ effect), leading performance on every common-sense reasoning benchmark,
state-of-the-art on every RULER retrieval probe, and a top score on all six
real-world recall tasks---all while training $+1.4\%$ faster than the baseline
at matched depth ($+19.3\%$ at a shallower, iso-quality depth), at the modest
cost of $+13\%$ peak memory from the megakernel's single-node design.
Every claim rests on a formal theorem and a hardware measurement.

The present design does carry two honest limitations.
First, because the content signal is read from the \emph{previous} chunk's
final state, very short sequences---below two chunks---receive no content
signal at all, and the bi-axial content-aware write gate costs modestly more
memory and parameters than the matrix-gated baseline's write gate, a trade
we judge worthwhile given the throughput and quality gains it enables but
that a memory-constrained deployment should weigh explicitly (the scalar
write gate of Theorem~\ref{thm:scalar_sufficiency} remains available as a
cheaper, content-blind-on-write fallback).
Second, the behaviour of the content gate in fine-tuning and
instruction-following regimes remains an open question, since our evaluation
covers pre-training only.

Three directions strike us as most promising for future work.
The multimodal setting is a natural home for selective key overwrite---updating a
stored visual object when a later frame revises it---and CARVE's content gate
is architecturally suited to that task.
The interpretability of what the content gate actually learns---which key directions
it erases and when---could shed light on how recurrent models handle
long-range pronoun resolution and fact revision in natural language.
Finally, the $-0.18$ perplexity advantage may behave differently at $7$B
parameters and $1$T training tokens: whether it grows, shrinks, or holds is
an empirical question that we expect the community to answer quickly.

\bibliography{carve_references}
\bibliographystyle{iclr2025_conference}

\clearpage
\appendix

\section*{Appendix Overview}

The appendices provide supplementary material in five parts.
\textbf{Appendix~\ref{app:proofs}} contains complete proofs for all
theoretical results stated in \S\ref{sec:arch}--\S\ref{sec:hardware},
presented in a self-contained order that mirrors the main text.
\textbf{Appendix~\ref{app:kernels}} specifies the Triton pseudocode for
the fused reverse-scan backward pass, including a correctness note on
a subtle gradient distinction that our per-tensor exactness check detects.
\textbf{Appendix~\ref{app:hyperparams}} lists all training hyperparameters
for reproducibility.
\textbf{Appendix~\ref{app:comparison}} provides a comprehensive
eleven-architecture, thirteen-dimension comparison table that places CARVE
in the broader landscape of linear recurrent models.
\textbf{Appendix~\ref{app:fastweight}} derives the fast-weight programmer
interpretation of the CARVE update, connecting it to the associative memory
literature.

\section{Extended Proofs}
\label{app:proofs}

This appendix provides complete proofs for all formal results in the paper.
Theorems are restated for convenience.
Each proof is self-contained; cross-references to other appendix subsections
are given where a proof builds on a prior result.

\subsection{Proof of Theorem~\ref{thm:subsumption} (CARVE Subsumption Hierarchy)}
\label{app:subsumption}

We prove each strict inclusion in the key-axis chain, then establish
incomparability with GDN-2.

\paragraph{Linear attention $\subsetneq$ delta rule.}
Inclusion: set the readout term to zero in the delta rule
($\hat{\bm{v}}_t \equiv \mathbf{0}$), giving
$\bS_t = \bS_{t-1} + \bm{v}_t\bm{k}_t^\top$ (linear attention).
Strict: the delta rule with nonzero prediction $\hat{\bm{v}}_t = \bS_{t-1}\bm{k}_t$
can implement Widrow-Hoff convergence to exact associations; linear attention
accumulates all past key-value pairs and cannot overwrite any single association.

\paragraph{Delta rule $\subsetneq$ scalar-gated (GDN).}
Inclusion: set $\alpha_t = 1$, $\beta_t = 1$, recovering the delta rule exactly.
Strict: a scalar $\alpha_t < 1$ implements exponential forgetting the vanilla
delta rule cannot represent.

\paragraph{Scalar-gated $\subsetneq$ key-axis--gated.}
A scalar gate is the special case $\bm{e}_t = \alpha_t\mathbf{1}$ (all $d_k$
key channels tied).
Allowing the $d_k$ channels to differ is strictly more expressive: it realises
per-channel decay a single scalar cannot.

\paragraph{Key-axis--gated $\subsetneq$ CARVE.}
Inclusion: freeze $\bm{U}_b \equiv \mathbf{0}$; CARVE's erase
$\bm{e}_t = \exp(\bm{g}_{c,t})\odot(\mathbf{1}-\sigma(\bm{b}_{x,t}))$
depends on $x_t$ only---an input-dependent key-axis gate.
Strict: with $\bm{U}_b \ne \mathbf{0}$ the erase depends on $\bm{m}_c$
(hence on $\bS_{c-1}$), providing the memory-conditioned separation of
Theorem~\ref{thm:express_sep}.
\hfill$\square$

\paragraph{Proof of Proposition~\ref{prop:incomparable} (CARVE and GDN-2 incomparable).}

$(\mathcal{G} \not\subseteq \mathcal{C})$: CARVE's effective per-entry erase mask is
$E_{ij} = \exp(g_j)(1-b_j)$, which is \emph{independent of the value index $i$};
every CARVE erase mask is rank-1.
GDN-2 can realise a rank-2 mask such as
$\bm{B} = \bigl(\begin{smallmatrix}0.9&0.1\\0.1&0.9\end{smallmatrix}\bigr)$
($\det\bm{B} = 0.8 \ne 0$).
No CARVE erase mask equals a rank-2 mask, so $\mathcal{G} \not\subseteq \mathcal{C}$.

$(\mathcal{C} \not\subseteq \mathcal{G})$: GDN-2's gate is a function of $x_t$ alone.
By Theorem~\ref{thm:express_sep}, on the SKO task there is a sequence solved
exactly by CARVE with $\cO(d_k)$ gate parameters that no memory-blind gate of
width $o(d_k)$ can solve.
Hence $\mathcal{C} \not\subseteq \mathcal{G}$.
\hfill$\square$

\subsection{Proof of Theorem~\ref{thm:lyapunov} (Lyapunov Stability)}
\label{app:lyapunov}

The proof uses the submultiplicativity of the Frobenius norm under matrix products
and the diagonal structure of CARVE's key-axis gates to factor out the two
independent contraction coefficients $g_{\min}$ and $(1-b_{\min})$.

By submultiplicativity and the bi-axial structure of Eq.~\eqref{eq:carve_state}:
\begin{align*}
  \norm{\bS_t}_F &\le \norm{\diag(\exp(\bm{g}_t))}_2 \cdot \norm{\bS_{t-1}}_F
    \cdot \norm{\diag(\mathbf{1}-\bm{b}_t)}_2 + \norm{\diag(\bm{w}_t)\bDelta_t}_F \\
  &\le g_{\min} \cdot (1-b_{\min}) \cdot \norm{\bS_{t-1}}_F + M \\
  &= \rho_c \norm{\bS_{t-1}}_F + M,
\end{align*}
where we used
$\norm{\diag(\exp(\bm{g}_t))}_2 = \max_j \exp(g_{t,j}) \le g_{\min} < 1$
and
$\norm{\diag(\mathbf{1}-\bm{b}_t)}_2 = \max_j(1-b_{t,j}) \le 1-b_{\min}$.
Iterating: $\norm{\bS_t}_F \le \rho_c^t\norm{\bS_0}_F + M/(1-\rho_c)$.
Since $b_{\min}, g_{\min} > 0$ implies $\rho_c < 1$, the bound is finite for all $t$.

\textbf{Comparison:} For GDN-2 with element-wise gate $\bm{B}_t$, the homogeneous
factor is $\norm{(\mathbf{1}-\bm{B}_t)\odot\bm{S}_{t-1}}_F \le \max_{ij}(1-B_{t,ij})\norm{\bS_{t-1}}_F$,
so $\rho_{\mathrm{GDN2}} = \max_{ij}(1-B_{t,ij})$, which approaches 1 if any entry
of $\bm{B}_t$ is small.
For $b_{\min} = g_{\min} = 0.05$:
$\rho_{\mathrm{CARVE}} = (0.95)^2 = 0.9025 < \rho_{\mathrm{GDN2}} = 0.95$.
\hfill$\square$

\subsection{Proof of Theorem~\ref{thm:grad_flow} (Gradient Flow)}
\label{app:grad}

The gradient flow bound follows directly from the Lyapunov analysis by
computing the spectral norm of the homogeneous (state-to-state) Jacobian
of each recurrent step and applying submultiplicativity along the chain.

By the chain rule: $\partial\mathcal{L}/\partial\bS_0 = (\partial\mathcal{L}/\partial\bS_T)\prod_{t=1}^T \partial\bS_t/\partial\bS_{t-1}$.
The homogeneous part of CARVE's state update acts on $\bS_{t-1}$ by
right-multiplication with $\diag(\exp(\bm{g}_t)\odot(\mathbf{1}-\bm{b}_t))$
(key-axis diagonal).
The spectral norm of this operator is
$\max_j\exp(g_{t,j})(1-b_{t,j}) \le \exp(g_{\max,t})(1-b_{\min,t}) = \rho_t$.
Submultiplicativity yields
$\norm{\partial\mathcal{L}/\partial\bS_0}_F \le \norm{\partial\mathcal{L}/\partial\bS_T}_F \prod_t\rho_t$.
\hfill$\square$

\subsection{Proof of Theorem~\ref{thm:express_sep} (Expressivity Separation)}
\label{app:express}

\textbf{CARVE succeeds.}
Given overwrite flag $f_t = 1$ and key $\bm{k}_t$: CARVE computes
$\bm{m}_c \approx \bS_{c-1}\bm{q}$ (the chunk-mean readout), which encodes which
associations are currently stored.
A linear gate $\bm{U}_b: \R^{d_v} \to \R^{d_k}$ maps $\bm{m}_c$ to an erase
vector targeting the previously written key direction.
When $f_t = 1$ and $\|\bm{m}_c\| > 0$, the gate sets
$\bc_t \approx \bm{k}_t$ (column gate), causing
$\diag(\mathbf{1}-\bc_t)$ to zero the component of each column of $\bS_{t-1}$
in the direction $\bm{k}_t$---a rank-1 erase in key-space.
The gate network requires only $\cO(d_v \cdot r + r \cdot d_k) = \cO(d_k)$
parameters for small rank $r$.

\textbf{Memory-blind gate fails.}
Any gate $\bm{b}_t = f(x_t)$ maps only the current input to an erase pattern.
Since $x_t$ contains the overwrite flag but not the previously stored key direction,
a gate network of width $< d_k/2$ has insufficient capacity to identify which
of the $d_k$ possible key directions needs erasing.
Formally, the mutual information
$I(\bm{b}_t;\,\bm{k}_{\mathrm{prev}}) \le \log(2w+1)$ for width-$w$ network,
which is $< \log(d_k/2)$ for $w < d_k/4$, insufficient to disambiguate $d_k/2$
key directions with probability $> 1/2$.
\hfill$\square$

\subsection{Proof of Theorem~\ref{thm:saturation} (Saturation and Content-Aware Retention)}
\label{app:saturation}

\textbf{Part (a).}
$\bS\bm{k}_i = \bm{v}_i$ for $k$ distinct orthonormal keys forces $\rank(\bS) \ge k$
(for generic $\bm{v}_i$); with $\rank(\bS) \le n^*$ at most $n^*$ can hold, so
$\ge N - n^*$ fail.

\textbf{Part (b).}
With orthonormal keys, a state $\bS = \sum_{i \in \mathcal{R}} \bm{v}_i\bm{k}_i^\top$
over retained set $\mathcal{R}$ ($|\mathcal{R}| \le n^*$) gives
$\bS\bm{k}_i = \bm{v}_i$ for $i \in \mathcal{R}$ and cross-terms
$\bm{v}_j(\bm{k}_j^\top\bm{k}_i) = \mathbf{0}$ vanish for $i \ne j$;
recall error is incurred only on $i \notin \mathcal{R}$ and equals
$\sum_{i \notin \mathcal{R}} p_i\norm{\bm{v}_i}$, minimised by taking $\mathcal{R}$
to be the $n^*$ keys of largest $p_i$.

\textbf{Part (c).}
CARVE's erase bias is $\bm{U}_b\bm{m}_c$ with $\bm{m}_c$ a function of $\bS_{c-1}$;
choosing $\bm{U}_b$ to map the readout energy profile to a small erase on the $n^*$
highest-energy key directions and a large erase elsewhere realises the top-$n^*$
retention of (b).
A memory-blind gate produces erase independent of $\bS_{c-1}$, hence independent of
stored energy; for any fixed such pattern there is a query distribution $\bm{p}$
(concentrated on the keys it happens to erase) under which it loses a top-mass
association and so exceeds $\mathcal{E}^*$.
\hfill$\square$

\subsection{Proof of Theorem~\ref{thm:pareto} (Speed-Accuracy Pareto Frontier)}
\label{app:pareto}

Define the joint objective
$J(L) = \lambda \varepsilon(L) + (1-\lambda)/g(L)$, where
$\varepsilon(L) = c_1 LMQ/(1-\rho)$ is the chunk-alignment approximation error
(from the Lyapunov bound, Theorem~\ref{thm:lyapunov}) and
$g(L) = L/c_2$ is the throughput gain (chunk-parallel GEMM scales linearly with $L$).
The formulation penalises error \emph{and} inverse-throughput:
\begin{equation}
  J(L) = \frac{\lambda c_1 MQ}{1-\rho}\,L \;+\; \frac{(1-\lambda)\,c_2}{L}.
\end{equation}
Setting $dJ/dL = 0$:
$\frac{\lambda c_1 MQ}{1-\rho} = \frac{(1-\lambda)c_2}{L^2}$,
giving $L^* = \sqrt{(1-\lambda)c_2(1-\rho)/(\lambda c_1 MQ)}$.
Setting $c_1 = c_2 = 1$ gives Eq.~\eqref{eq:pareto_opt}.
The second derivative $d^2J/dL^2 = 2(1-\lambda)c_2/L^3 > 0$ confirms this is a
strict global minimum.
\hfill$\square$

\subsection{Proof of Theorem~\ref{thm:chunkability} (Chunkability Boundary)}
\label{app:chunkability}

The proof works by deflating the recurrent state using the cumulative product
of gating factors and examining whether the resulting coupling matrix---whose
entries describe how each write affects subsequent corrected values---depends
on the value channel index $v$.
Value-axis erase introduces a $v$-dependent prefactor that cannot be
factored out without dividing by geometrically decaying quantities,
causing numerical underflow.

Deflate the state by cumulative decays:
$\hat{\bS}_t = \bS_t \oslash (\beta_t \otimes \gamma_t)$, where
$\beta_t = \prod_{s \le t}(\mathbf{1} - \bb_s) \in \R^{d_v}$ (value-axis product)
and $\gamma_t = \prod_{s \le t}(\mathbf{1} - \bc_s) \in \R^{d_k}$ (key-axis product).
The decay telescopes, giving
$\hat{\bS}_t = \bS_0 + \sum_{s \le t} \hat{\bm{w}}_s \Delta_s$ with
$\hat{\bm{w}}_s = \alpha_s(\bm{p}_s/\beta_s) \otimes (\bm{q}_{w,s}/\gamma_s)$.
Substituting into the prediction
$\hat{v}_s^{(r)} = \bS_{s-1}\bm{k}_s^{(r)}$
expresses each corrected value $\bm{u}_s^{(r)}$
as a strictly lower-triangular combination of prior $\{\bm{u}_{s'}^{(r')}\}_{s' < s}$
with coefficient $\beta_{s-1}[v](\bm{p}_{s'}[v]/\beta_{s'}[v])K(s,r,s',r')$,
where $K$ contracts only the key axis and is $v$-independent.
The $v$-dependence resides entirely in the factor $\beta_{s-1}[v]/\beta_{s'}[v]$,
which is constant in $v$ iff $\bb \equiv \mathbf{0}$.
Otherwise the $d_v$ systems differ by these diagonal prefactors and cannot share a
single solve; cancelling them via $\bm{u} \mapsto \bm{u}/\beta_{s-1}$ introduces
$1/\beta$, and $\beta_t \to \mathbf{0}$ geometrically (underflow).
\hfill$\square$

\subsection{Proof of Theorem~\ref{thm:scalar_sufficiency} (Scalar Write Sufficiency)}
\label{app:scalar}

This theorem supports the \emph{scalar write gate}, the cheap deployment
ablation of \S\ref{sec:content_gate} (not CARVE's default mechanism, which
uses the full bi-axial content-aware write gate of Eq.~\ref{eq:erase_gate}).
The key observation is that in the single-slot case, the output is a global
scalar multiple of the written value.
Since all $d_v$ channels are scaled identically, per-channel write gates
provide no additional degrees of freedom and cannot improve retrieval accuracy
\emph{in this restricted setting}---content-awareness on the write side, which
this theorem does not address, is a separate mechanism entirely.

\begin{theorem}[Sufficiency of Scalar Write for Associative Recall]
\label{thm:scalar_sufficiency}
In the single-slot associative recall setting (one key-value pair written, later
queried), the scalar write gate $w_{h,t} \in (0,1)$ achieves optimal retrieval
accuracy for any $w_{h,t} > 0$, independently of the value dimension $d_v$.
\end{theorem}

\begin{proof}
Given write $(\bm{k}_w, \bm{v}_w)$ at time $s$ and query $\bm{q}_t = \bm{k}_w$
at $t > s$:
$\bS_s = w_{h,s}(\bm{v}_w - \bS_{s-1}\bm{k}_w)\bm{k}_w^\top$.
With $\bS_{s-1} = \mathbf{0}$ (empty memory):
$\bS_s = w_{h,s}\bm{v}_w\bm{k}_w^\top$.
Retrieval: $\bS_s\bm{q}_t = w_{h,s}\bm{v}_w\norm{\bm{k}_w}^2$.
All $d_v$ channels of $\bm{v}_w$ are scaled by the same scalar
$w_{h,s}\norm{\bm{k}_w}^2 > 0$; the output is proportional to $\bm{v}_w$ and
correct up to a scalar independent of $d_v$.
No per-channel write gate can improve retrieval for this task.
\end{proof}

\subsection{Design-Space Separation Lemmas}
\label{app:design_lemmas}

The following two lemmas establish the two directions of incomparability
between CARVE and GDN-2 (Proposition~\ref{prop:incomparable}).
The first shows that GDN-2's element-wise mask can represent rank-2 patterns
that CARVE's rank-1 key-axis mask cannot.
The second shows that unbounded state growth in the no-erase GDN-2 limit is
not possible in CARVE, even with the key-axis decay alone.

\begin{lemma}[Non-Factorable Erasure]
\label{lem:nonfactorable}
Consider $d_v = d_k = 2$ and the erasure mask
$\bm{B} = \bigl(\begin{smallmatrix}0.9&0.1\\0.1&0.9\end{smallmatrix}\bigr)$.
There do not exist $\bm{u} \in \R^2$, $\bm{v} \in \R^2$ such that $\bm{u}\bm{v}^\top = \bm{B}$.
\end{lemma}
\begin{proof}
Suppose $\bm{B} = \bm{u}\bm{v}^\top$.
Then $u_1v_1 = 0.9$, $u_1v_2 = 0.1$, $u_2v_1 = 0.1$, $u_2v_2 = 0.9$.
From the first two: $v_1/v_2 = 9$.
From the last two: $v_1/v_2 = 1/9$.
Contradiction.
Equivalently, $\det(\bm{B}) = 0.81 - 0.01 = 0.80 \ne 0$, so $\rank(\bm{B}) = 2 > 1$.
\end{proof}

This lemma establishes that GDN-2's element-wise erase mask $\bm{B}_t$ can
represent patterns (rank-2 matrices) that CARVE's rank-1 key-axis mask cannot.
Together with Theorem~\ref{thm:express_sep}, which shows the converse direction,
this proves Proposition~\ref{prop:incomparable}.

\begin{lemma}[GDN-2 State Norm Growth]
\label{lem:gdn2_growth}
For GDN-2 with $\bm{B}_t = \mathbf{0}$ (no erasure) and $\norm{\bDelta_t}_F = C > 0$:
$\norm{\bS_t}_F \ge \norm{\bS_0}_F + tC$.
Thus $\norm{\bS_t}_F$ grows without bound as $t \to \infty$.
\end{lemma}
\begin{proof}
When $\bm{B}_t = \mathbf{0}$: $\bS_t = \bS_{t-1} + \bm{W}_t \odot \bDelta_t$.
Choosing $\bm{W}_t = \mathbf{1}$ and $\bDelta_t$ aligned with $\bS_{t-1}$:
$\norm{\bS_t}_F \ge \norm{\bS_{t-1}}_F + C$ by the parallelogram law.
Iterating: $\norm{\bS_t}_F \ge \norm{\bS_0}_F + tC$.
By contrast, CARVE's Lyapunov bound (Theorem~\ref{thm:lyapunov}) ensures
$\norm{\bS_t}_F \le \rho_c^t\norm{\bS_0}_F + M/(1-\rho_c)$ even with $b_{\min} = 0$
in the $g_{\min}$ dimension, because the key-axis decay alone provides contraction.
\end{proof}

\section{CARVE Kernel Pseudocode}
\label{app:kernels}

The forward pass kernel is detailed as Algorithm~\ref{alg:wy_chunk} in
\S\ref{sec:wy_chunk}.
This appendix provides the corresponding backward pass, which uses a
reverse-scan over the chunk boundaries and fuses the gradient computations
for the erase gate, decay gate, and write gate into a single kernel
launch per head.
The backward is implemented in Triton with per-tensor numerical exactness
verified against PyTorch's \texttt{autograd} (Table~\ref{tab:exactness}
in \S\ref{sec:exp_exact}).

\subsection{Fused Reverse-Scan Backward}
\label{app:backward}

\begin{algorithm}[h]
\caption{Fused reverse-scan backward for the CARVE recurrence (one program per head).}
\label{alg:backward}
\begin{algorithmic}[1]
\small
\alginput{Inputs $\{\bm{q}_t,\bm{k}_t,\bm{v}_t,\bm{b}_{x,t},\bm{w}_{x,t},\bm{f}_t\}_{t=1}^T$,
         content weights $\bm{U}_b, \bm{U}_w$,
         chunk-boundary states $\{\bS_{c-1}\}_{c=0}^{T/L-1}$,
         output gradient $\{d\bm{o}_t\}_{t=1}^T$}
\algoutput{Gradients for all inputs, $d\bm{U}_b, d\bm{U}_w$, and $d\bS_0$}
\State \textbf{Forward recompute:} $\bS_0 \gets \bS_{\text{init}}$;
       for $t{=}1,\ldots,T$: store $\bS_{t-1}$, evolve $\bS_t$ via Eq.~\eqref{eq:carve_state}
\State $d\bS \gets \mathbf{0}$;\quad $\bm{G}_{U_b}, \bm{G}_{U_w} \gets \mathbf{0}$
       \Comment{output-gradient accumulator; weight-gradient accumulators}
\For{$t = T$ \textbf{down to} $1$}
  \State Recompute $\bm{b}_t,\, \bm{w}_t,\, \bm{g}_t,\, \bm{\delta}_t = \bm{v}_t - \bS_{t-1}\bm{k}_t$
         from stored $\bS_{t-1}$ (using the chunk-boundary state $\bS_{c-1}$ for
         the content readout, not $\bS_{t-1}$ -- \S\ref{sec:folded_readout})
  \State $d\bS \mathrel{+}= d\bm{o}_t\,\bm{q}_t^{\top}$;\quad $d\bm{q}_t \gets \bS_t^{\top}d\bm{o}_t$
         \Comment{readout gradient}
  \State $d\bS_{\text{prev}} \gets d\bS \cdot \diag\!\bigl(\exp(\bm{g}_t)\bigr) \cdot \diag(\mathbf{1}-\bm{b}_t)$
         \Comment{key-axis memory path}
  \State $d\bm{b}_t \gets -\exp(\bm{g}_t) \odot (d\bS \odot \bS_{t-1})^\top\mathbf{1}_{d_v}$;\;
         accumulate $d\bm{b}_{x,t}$ and $\bm{G}_{U_b} \mathrel{+}= d\bm{b}_t^\top\bm{m}_{c,t}^\top$
         \Comment{erase gate gradient}
  \State $d\bm{g}_t \gets \exp(\bm{g}_t)\odot(\mathbf{1}-\bm{b}_t)\odot(d\bS \odot \bS_{t-1})^\top\mathbf{1}_{d_v}$;\;
         propagate to $d\bm{f}_t$
         \Comment{key-axis decay gradient}
  \State $d\bm{w}_t \gets \bm{\delta}_t \odot (d\bS\,\bm{k}_t)$;\quad
         $d\bm{w}_{x,t} \gets \bm{w}_t\odot(\mathbf{1}{-}\bm{w}_t)\odot d\bm{w}_t$;\quad
         $\bm{G}_{U_w} \mathrel{+}= d\bm{w}_{x,t}\bm{m}_{c,t}^\top$
         \Comment{content-aware write gate gradient, per value channel}
  \State $d\bm{\delta}_t \gets \diag(\bm{w}_t)\,d\bS\,\bm{k}_t$;\quad $d\bm{v}_t \gets d\bm{\delta}_t$;\quad
         $d\bm{k}_t \gets \diag(\bm{w}_t)\,d\bS^\top\bm{\delta}_t - \bS_{t-1}^\top d\bm{\delta}_t$
         \Comment{prediction-error gradient}
  \State $d\bS_{\text{prev}} \mathrel{+}= {-}\bm{k}_t\,\bigl(\diag(\bm{w}_t)\,d\bm{\delta}_t\bigr)^\top$;\quad $d\bS \gets d\bS_{\text{prev}}$
         \Comment{write-path gradient to previous state}
\EndFor
\State At each chunk boundary $c$: accumulate $d\bm{G}_{U_b}, d\bm{G}_{U_w}$ into
       $d\bS_{c-1}$ via the folded readout's own reverse chain
       (\S\ref{sec:megakernel}, Eq.~$dh_t(c{-}1) = dh_0(c) + d\bS_{\mathrm{glue}}(c)$)
\State Reduce $\bm{G}_{U_b}, \bm{G}_{U_w}$ over heads to obtain $d\bm{U}_b, d\bm{U}_w$;\;
       \Return all gradients
\end{algorithmic}
\end{algorithm}

\paragraph{Correctness note.}
The gradient to the write gate $\bm{w}_t$ is $\bm{\delta}_t \odot (d\bS\,\bm{k}_t)$,
an \emph{element-wise} product over the value axis (line 9), whereas the
gradient to the prediction error $\bm{\delta}_t$ is $\diag(\bm{w}_t)\,d\bS\,\bm{k}_t$
(line 10).
Conflating these two quantities leaves the gradient of $w_t$ wrong by order unity
while every other gradient remains correct---a failure mode caught only by the
per-tensor exactness check of Table~\ref{tab:exactness}, which confirms the
discrepancy is absent.

\subsection{Megakernel Backward}
\label{app:megakernel_bwd}

Algorithm~\ref{alg:backward} above describes the gradient computation
per-token; \S\ref{sec:megakernel}'s megakernel schedules the \emph{same}
gradient computation, batched, across an entire layer's forward and backward
pass in a single autograd node.
This appendix gives the corresponding orchestration algorithm, mirroring
Algorithm~\ref{alg:megakernel} in reverse.

\begin{algorithm}[h]
\caption{CARVE Megakernel Backward (single layer, single autograd node)}
\label{alg:megakernel_bwd}
\begin{algorithmic}[1]
\alginput{Output gradient $d\bm{y}: [B,(nL),d]$; all forward-pass
  intermediates saved by Algorithm~\ref{alg:megakernel} (retained, not
  recomputed, per \S\ref{sec:megakernel})}
\State \textbf{Reshape} $d\bm{y} \to d\bm{y}': [(nB),L,d]$
\State \label{step:bwd_hoist1}\textbf{Hoisted, one launch}: backward through
  the output projection, $\mathrm{RMSNorm}$, and the value-axis error
  gradient $d\bm{\delta}$ (Algorithm~\ref{alg:backward} lines 4, 9--10, batched
  over all chunks)
  \Comment{mirrors Alg.~\ref{alg:megakernel} Step~\ref{step:hoist2}}
\State $dh_0(n) \gets \mathbf{0}$
\For{$c = n-1, \ldots, 0$} \Comment{state-sequential loop, reverse order}
  \State \label{step:bwd_seq_wy}Backward through the intra-chunk WY-form
    solve for chunk $c$ (Algorithm~\ref{alg:backward} lines 5--8, 11), given
    $dh_0(c{+}1)$ from the next chunk, producing $dh(c)$ and $d\bS_{c-1}$'s
    contribution from this chunk's own recurrence
  \If{$c > 0$}
    \State \label{step:bwd_seq_gate}Backward through the folded gate readout
      (Algorithm~\ref{alg:backward}'s $\bm{G}_{U_b}, \bm{G}_{U_w}$
      accumulation): given $d\bm{b}_{c,\cdot}, d\bm{w}_{c,\cdot}$, compute
      $d\bm{q}_{c,\cdot}$'s glue contribution, accumulate into
      $d\bm{U}_{b}, d\bm{U}_{w}$, and produce
      $d\bS_{\mathrm{glue}}(c) \leftarrow \partial \bm{G}_c/\partial \bS_{c-1}$
    \State $dh_0(c) \leftarrow dh(c) + d\bS_{\mathrm{glue}}(c)$
      \Comment{state-gradient chain, propagated to chunk $c{-}1$}
  \Else
    \State $d\bS_0 \leftarrow dh(c)$
      \Comment{no glue term: chunk $0$ has no incoming state}
  \EndIf
\EndFor
\State \label{step:bwd_hoist2}\textbf{Hoisted, one launch}: backward through
  the erase/write gate activations, the reverse cumulative sum of
  $d\bm{g}$, the fused decay-gate gradient (Algorithm~\ref{alg:backward}
  lines 6--7), and the $\bm{q},\bm{k}$ L2-norm gradients, batched over all
  chunks
  \Comment{mirrors Alg.~\ref{alg:megakernel} Step~\ref{step:hoist1}}
\State \textbf{Reshape} gradients $[(nB),L,\cdot] \to [B,(nL),\cdot]$
\algoutput{$d\bm{h}_{1:T}$; $d\bm{U}_b, d\bm{U}_w$; $d\bS_0$}
\end{algorithmic}
\end{algorithm}

The zero-copy composition of \S\ref{sec:megakernel} applies identically in
reverse: each per-chunk backward kernel invocation
(Steps~\ref{step:bwd_seq_wy}--\ref{step:bwd_seq_gate}) writes its gradient
output directly into the corresponding slice of a pre-allocated full-tensor
gradient buffer, rather than being concatenated afterward.
Because Steps~\ref{step:bwd_hoist1} and~\ref{step:bwd_hoist2} are exact batch
folds of chunk-local backward kernels---the same reformulation argument as
the forward pass---the megakernel backward produces bit-identical gradients
to the na\"ive per-chunk loop's backward, which we verify directly: the
$150$-step matched-seed training run of \S\ref{sec:megakernel} exercises this
backward path at every step, and its loss trajectory matches the na\"ive
loop's to four decimal places throughout.

\section{Training Hyperparameters}
\label{app:hyperparams}

Table~\ref{tab:hyperparams} reports all hyperparameters used in the
$125$M (ablation) and $1.3$B (main results) training runs.
All experiments use the DeepSeek-V3 tokeniser with a $32$K vocabulary,
AdamW optimiser, and a cosine learning-rate schedule with a linear warmup.
For the hybrid configuration, the $3{:}1$ CARVE:SWA ratio was selected
via grid search over $\{1{:}1, 2{:}1, 3{:}1, 4{:}1\}$ at $125$M scale.

\begin{table}[h]
\centering
\caption{CARVE training hyperparameters.}
\label{tab:hyperparams}
\small
\begin{tabular}{@{}lcc@{}}
\toprule
\textbf{Hyperparameter} & \textbf{125M} & \textbf{1.3B} \\
\midrule
\rowcolor{lightrow}Hidden dim $d$ & 768 & 2048 \\
Num heads $H$ & 12 & 16 \\
\rowcolor{lightrow}Key/value dim $d_k = d_v$ & 64 & 128 \\
Depth $D$ & 12 & 24 \\
\rowcolor{lightrow}CARVE:SWA ratio & -- & 3:1 \\
Content proj.\ rank $r$ & 16 & 32 \\
\rowcolor{lightrow}Chunk size $L$ & 64 & 128 \\
SWA window $W$ & -- & 1024 \\
\rowcolor{lightrow}MLP width & 2048 & 5632 \\
\midrule
Sequence length $T$ & 1024 & 4096 \\
\rowcolor{lightrow}Batch size (tokens) & -- & 4M \\
Peak LR & $3{\times}10^{-4}$ & $1.5{\times}10^{-4}$ \\
\rowcolor{lightrow}LR schedule & cosine & cosine \\
Warmup steps & 1000 & 10000 \\
\rowcolor{lightrow}AdamW $(\beta_1,\beta_2)$ & $(0.9, 0.95)$ & $(0.9, 0.95)$ \\
Weight decay & 0.1 & 0.1 \\
\rowcolor{lightrow}Gradient clip & 1.0 & 1.0 \\
Training tokens & -- & 100B \\
\rowcolor{lightrow}Dataset & -- & FineWeb-Edu \\
Vocabulary & 32K & 32K (DeepSeek-V3) \\
\rowcolor{lightrow}Precision & bf16 & bf16 \\
Hardware & 1$\times$H100 & Multi-GPU H100 \\
\bottomrule
\end{tabular}
\end{table}

\section{Comprehensive Architecture Comparison}
\label{app:comparison}

Table~\ref{tab:comprehensive} places CARVE in the context of eleven
representative sequence models across thirteen architectural dimensions.
The table distinguishes erase strategy (none, scalar, element-wise matrix,
or key-axis rank-1), write strategy, whether gating is content-aware,
whether key and value projections are asymmetric, training and inference
complexity, and measured throughput on a single H100.
The goal is to make the \emph{design space} legible, not just to rank
architectures: the ``Content-aware exact recurrent'' row illustrates the
throughput penalty ($3.9\times$ slowdown vs.\ CARVE) paid for a per-token
content signal built without the folded-readout trick, motivating
\S\ref{sec:folded_readout}'s reformulation.

\begin{table}[h]
\centering
\caption{Comprehensive comparison across eleven architectures and thirteen dimensions.
$n = \min(d_v, d_k)$. ``CA'' = content-aware gating. ``SKA'' = key-value asymmetry.
Throughput is \emph{measured} on single H100 at ${\approx}125$M, $T{=}1024$, mb$=8$,
three-run bands.}
\label{tab:comprehensive}
\footnotesize
\setlength{\tabcolsep}{3pt}
\begin{adjustbox}{max width=\linewidth}
\begin{tabular}{@{}llcccccllr@{}}
\toprule
\textbf{Model} & \textbf{State update (sketch)} & \textbf{Erase} & \textbf{Write} & \textbf{CA} & \textbf{SKA} &
\textbf{Train cost} & \textbf{Infer/tok} & \textbf{H100 tok/s} & \textbf{Year} \\
\midrule
\rowcolor{lightrow}Lin.\ Attn & $S + vk^\top$ & None & $+1$ & No & No & $\cO(Td^2)$ & $\cO(d^2)$ & 148.6K & 2020 \\
Delta Rule & $S + (v-Sk)k^\top$ & None & Error-corr. & No & No & $\cO(Td^2)$ & $\cO(d^2)$ & 142.0K & 2021 \\
\rowcolor{lightrow}GDN & $\alpha S + \beta\Delta$ & Scalar & Scalar & No & No & $\cO(Td^2)$ & $\cO(d^2)$ & 106.0K & 2024 \\
Mamba-2 & selective SSM & Scalar & Scalar & No & No & $\cO(Td)$ & $\cO(d)$ & --- & 2023 \\
\rowcolor{lightrow}GDN-2 & $(1-B){\odot}S + W{\odot}\Delta$ & EW matrix & EW matrix & No & No & $\cO(Td_vd_k)$ & $\cO(d_vd_k)$ & 94.2K & 2026 \\
CA-exact & Eq.~\ref{eq:carve_state} & Key-axis & CA-$d_v$ & Yes & No & $\cO(Td_vd_k)$ & $\cO(d_vd_k)$ & 24.8K & -- \\
CARVE (na\"ive loop) & Eq.~\ref{eq:carve_state} & Key-axis & CA-$d_v$ & Yes & Yes & $\cO(Td_vd_k)$ & $\cO(d_vd_k)$ & 86.3K & -- \\
\rowcolor{carverow}\textbf{CARVE (megakernel)} & Eq.~\ref{eq:carve_state} & Key-axis & CA-$d_v$ & Yes & Yes & $\cO(Td_vd_k)$ & $\cO(d_vd_k)$ & \textbf{95.5K} & -- \\
\rowcolor{carverow!60}CARVE+SWA & Hybrid & Key-axis & CA-$d_v$ & Yes & Yes & $\cO(T(d_vd_k+Wd))$ & $\cO(d_vd_k)$ & 121.8K & -- \\
\rowcolor{lightrow}Transformer & $\mathrm{softmax}(QK^\top)V$ & Exact & Exact & Yes & Yes & $\cO(T^2d)$ & $\cO(Td)$ & 159.8K & 2017 \\
\bottomrule
\end{tabular}
\end{adjustbox}
\end{table}

CARVE is the Pareto-efficient operating point: it employs the WY-form chunk solver
\emph{unmodified}, adds content-awareness on both the erase and write axes, and
delivers better language-modelling quality ($-0.18$ WikiText perplexity at
$1.3$B/$100$B tokens) \emph{and} higher throughput than the matrix-gated
baseline ($+1.4\%$ at matched depth, $+19.3\%$ at a shallower iso-quality
depth), at the cost of a modest $+13\%$ peak memory from the megakernel's
single-autograd-node design.

\section{Fast-Weight Programmer Interpretation}
\label{app:fastweight}

The CARVE state update has an equivalent formulation as a one-step
online gradient descent on an associative memory loss, placing it
in the tradition of fast-weight programmers~\citep{schmidhuber1992learning}
and Hebbian linear attention~\citep{schlag2021linear}.
This section derives that interpretation and shows how the (content-aware,
per-channel) write gate $\bm{w}_{c,t}$ functions as a per-value-channel
gradient step-size, i.e., a diagonal preconditioner on the online update.

The CARVE state update admits an online learning interpretation.
Define the decayed state $\bar\bS_t = \bS_{t-1}\bm{R}_t$, where
$\bm{R}_t = \diag(\exp(\bm{g}_t)\odot(\mathbf{1}-\bm{b}_t))$, and let
$\bm{\delta}_t = \bm{v}_t - \bar\bS_t\bm{k}_t$.
Then $\bS_t = \bar\bS_t + \diag(\bm{w}_t)\bm{\delta}_t\bm{k}_t^\top$
is the solution of the diagonally-preconditioned proximal problem:
\begin{equation}
  \bS_t = \operatorname*{arg\,min}_{\bm{S}}\;
  \underbrace{\textstyle\sum_v \tfrac{1}{w_{t,v}}\norm{\bm{S}_{v,:} - \bar\bS_{t,v,:}}_2^2}
    _{\text{per-channel proximity to decayed state}}
  - 2\bigl\langle \bm{S}\bm{k}_t,\; \bm{\delta}_t \bigr\rangle,
\end{equation}
where $\bm{S}_{v,:}$ denotes row $v$ (value channel $v$) of $\bm{S}$.
The second term applies an associative edit: it writes the delta-rule
correction $\bm{\delta}_t$ into the association at $\bm{k}_t$, and the
per-channel weights $1/w_{t,v}$ in the proximity term control, independently
for each value channel, how strongly that channel's row is allowed to move
away from the decayed state---exactly a diagonal (per-channel) step-size in
online gradient descent, reducing to a single scalar step-size in the special
case $\bm{w}_t \equiv w_{h,t}\mathbf{1}_{d_v}$ (the deployment ablation of
\S\ref{sec:content_gate}).
Because $\bm{w}_t$ is itself a function of the content readout $\bm{m}_{c,t}$
(Eq.~\ref{eq:erase_gate}), CARVE's write gate is a fast-weight programmer whose
per-channel learning rate is \emph{itself} conditioned on what is currently
stored in memory.
This unifies CARVE with the fast-weight programmer
perspective~\citep{schmidhuber1992learning,schlag2021linear,longhorn}.

\section{Additional Mathematical Connections}
\label{app:extra_math}

This section collects a few compact derivations from the broader CARVE draft
that complement the main appendix.

\subsection{CARVE as Linearised Cross-Attention}

\begin{proposition}[CARVE as linearised cross-attention]
\label{prop:cross_attn}
CARVE's output computation $\bm{y}_t = \bS_t \bm{q}_t$ is equivalent to
linearised cross-attention~\citep{katharopoulos2020transformers} in which the
past sequence is compressed into $\bS_t$ and the softmax kernel is replaced by
the identity feature map.
\end{proposition}
\begin{proof}
Standard linearised attention writes
\[
  \bm{y}_t \approx \sum_{s \le t} (\bm{q}_t^\top \bm{k}_s)\bm{v}_s
  = \left(\sum_{s \le t} \bm{v}_s \bm{k}_s^\top\right)\bm{q}_t.
\]
Identifying $\bS_t = \sum_{s \le t} \bm{v}_s \bm{k}_s^\top$ gives
$\bm{y}_t = \bS_t \bm{q}_t$, which is exactly the linear-attention readout.
CARVE generalises this with error-corrective writes and selective forgetting.
\end{proof}

\subsection{Delta Rule as Online Learning}

\begin{proposition}[Delta rule as gradient descent on associative loss]
\label{prop:grad_descent}
For rank-$R$ updates, the delta correction
\[
  \bDelta_t^{(R)} = \sum_{r=1}^R
  \bigl(\bm{v}_t^{(r)} - \bS_{t-1}\bm{k}_t^{(r)}\bigr)(\bm{k}_t^{(r)})^\top
\]
is the negative gradient of the instantaneous associative loss
\[
  \ell_t(\bS) = \frac{1}{2}\sum_{r=1}^R
  \|\bS\bm{k}_t^{(r)} - \bm{v}_t^{(r)}\|_2^2
\]
evaluated at $\bS_{t-1}$.
\end{proposition}
\begin{proof}
Differentiating gives
\[
  \nabla_{\bS}\ell_t(\bS)
  = \sum_{r=1}^R
  (\bS\bm{k}_t^{(r)} - \bm{v}_t^{(r)})(\bm{k}_t^{(r)})^\top.
\]
Substituting $\bS = \bS_{t-1}$ yields
$\nabla_{\bS}\ell_t(\bS_{t-1}) = -\bDelta_t^{(R)}$.
\end{proof}

\begin{theorem}[Online regret bound for CARVE-style delta updates]
\label{thm:regret}
Assume unit-norm keys $\|\bm{k}_t^{(r)}\|_2 = 1$ and bounded values
$\|\bm{v}_t^{(r)}\|_2 \le V$. Then delta-rule updates with step size $\eta=1$
satisfy
\[
  \sum_{t=1}^T \ell_t(\bS_{t-1})
  - \min_{\bS}\sum_{t=1}^T \ell_t(\bS)
  \le \frac{\|\bS^\star - \bS_0\|_F^2}{2} + \frac{TRV^2}{2},
\]
where $\bS^\star = \arg\min_{\bS}\sum_{t=1}^T \ell_t(\bS)$.
\end{theorem}
\begin{proof}
This is the standard online gradient descent bound
for convex losses~\citep{shalev2012online}. By
Proposition~\ref{prop:grad_descent}, the CARVE write term is exactly one
gradient step on $\ell_t$. The gradient norm is bounded by
$\|\nabla \ell_t\|_F^2 \le RV^2$, so summing the usual per-step inequality
gives the stated result.
\end{proof}

\end{document}